\documentclass[10pt,twocolumn,letterpaper]{article}

\usepackage{iccv}
\usepackage{times}
\usepackage{epsfig}
\usepackage{graphicx}
\usepackage{amsmath}
\usepackage{amssymb}
\usepackage{amsthm}
\usepackage{algorithm}
\usepackage{algorithmicx}
\usepackage{algpseudocode}
\usepackage{multirow}
\usepackage{booktabs}
\usepackage{longtable}
\usepackage{wrapfig}
\usepackage[accsupp]{axessibility}
\usepackage{authblk}
\usepackage{tocvsec2}
\usepackage{float}

\makeatletter
\renewcommand\AB@affilsepx{ \protect\Affilfont}
\makeatother
\newcommand*{\thead}[1]{%
\multicolumn{1}{c}{\bfseries\begin{tabular}{@{}c@{}}#1\end{tabular}}}

\usepackage[pagebackref=true,breaklinks=true,letterpaper=true,colorlinks,bookmarks=false]{hyperref}

\iccvfinalcopy % *** Uncomment this line for the final submission

\def\iccvPaperID{11742} % *** Enter the ICCV Paper ID here

\newtheorem{defn}{Definition}[section]
\newtheorem{thm}{Theorem}[section]

\newcommand{\nocontentsline}[3]{}
\newcommand{\tocless}[2]{\bgroup\let\addcontentsline=\nocontentsline#1{#2}\egroup}

\ificcvfinal\fi

\begin{document}

%%%%%%%%% TITLE
\title{Building a Winning Team: Selecting Source Model Ensembles using a Submodular Transferability Estimation Approach}

% \author{Vimal K B\\
% IIT Hyderabad\\
% \and

\makeatletter
\newcommand\email[2][]%
   {\newaffiltrue\let\AB@blk@and\AB@pand
      \if\relax#1\relax\def\AB@note{\AB@thenote}\else\def\AB@note{\relax}%
        \setcounter{Maxaffil}{0}\fi
      \begingroup
        \let\protect\@unexpandable@protect
        \def\thanks{\protect\thanks}\def\footnote{\protect\footnote}%
        \@temptokena=\expandafter{\AB@authors}%
        {\def\\{\protect\\\protect\Affilfont}\xdef\AB@temp{#2}}%
         \xdef\AB@authors{\the\@temptokena\AB@las\AB@au@str
         \protect\\[\affilsep]\protect\Affilfont\AB@temp}%
         \gdef\AB@las{}\gdef\AB@au@str{}%
        {\def\\{, \ignorespaces}\xdef\AB@temp{#2}}%
        \@temptokena=\expandafter{\AB@affillist}%
        \xdef\AB@affillist{\the\@temptokena \AB@affilsep
          \AB@affilnote{}\protect\Affilfont\AB@temp}%
      \endgroup
       \let\AB@affilsep\AB@affilsepx
}
\makeatother

\author[1]{Vimal K B*}
\author[1,3]{Saketh Bachu*}
\author[1]{Tanmay Garg}
\author[2]{Niveditha Lakshmi Narasimhan}
\author[2]{Raghavan Konuru}
\author[1]{Vineeth N Balasubramanian}
\affil[1]{Indian Institute of Technology, Hyderabad } 
\affil[2]{KLA }
\affil[3]{University of California, Riverside \newline}
\email{* Equal Contribution. Corresponding author: \nolinkurl{vimalkb96@gmail.com}}

% \author[1]{Vimal K B\thanks{vimalkb96@gmail.com (equal contribution)}}
% \author[1,3]{Saketh Bachu\thanks{saketh7000@gmail.com (equal contribution)}}
% \author[1]{Tanmay Garg \thanks{gargtanmay1@gmail.com}}
% \author[1]{Vineeth N Balasubramanian\thanks{vineethnb@cse.iith.ac.in}}
% \author[2]{Niveditha Narasimhan\thanks{niveditha@kla.com}}
% \author[2]{Raghavan\thanks{raghavan@kla.com}}
% \affil[1]{Indian Institute of Technology, Hyderabad}
% \affil[2]{KLA}
% \affil[3]{University of California, Riverside}

\maketitle
\ificcvfinal\thispagestyle{empty}\fi

%%%%%%%%% ABSTRACT
\begin{abstract}
% \vspace{-20pt}
Estimating the transferability of publicly available pre-trained models to a target task has assumed an important place for transfer learning tasks in recent years. Existing efforts propose metrics that allow a user to choose one model from a pool of pre-trained models without having to fine-tune each model individually and identify one explicitly. With the growth in the number of available pre-trained models and the popularity of model ensembles, it also becomes essential to study the transferability of multiple-source models for a given target task. The few existing efforts study transferability in such multi-source ensemble settings using just the outputs of the classification layer and neglect possible domain or task mismatch. Moreover, they overlook the most important factor while selecting the source models, viz., the cohesiveness factor between them, which can impact the performance and confidence in the prediction of the ensemble. To address these gaps, we propose a novel Optimal tranSport-based suBmOdular tRaNsferability metric (OSBORN) to estimate the transferability of an ensemble of models to a downstream task. OSBORN collectively accounts for image domain difference, task difference, and cohesiveness of models in the ensemble to provide reliable estimates of transferability. We gauge the performance of OSBORN on both image classification and semantic segmentation tasks. Our setup includes 28 source datasets, 11 target datasets, 5 model architectures, and 2 pre-training methods. We benchmark our method against current state-of-the-art metrics MS-LEEP and E-LEEP, and outperform them consistently using the proposed approach.
\end{abstract}

%%%%%%%%% BODY TEXT
\vspace{-3pt}
\tocless\section{Introduction}
\vspace{-2pt}

\begin{figure*}
  \centering
\includegraphics[width=\textwidth]
{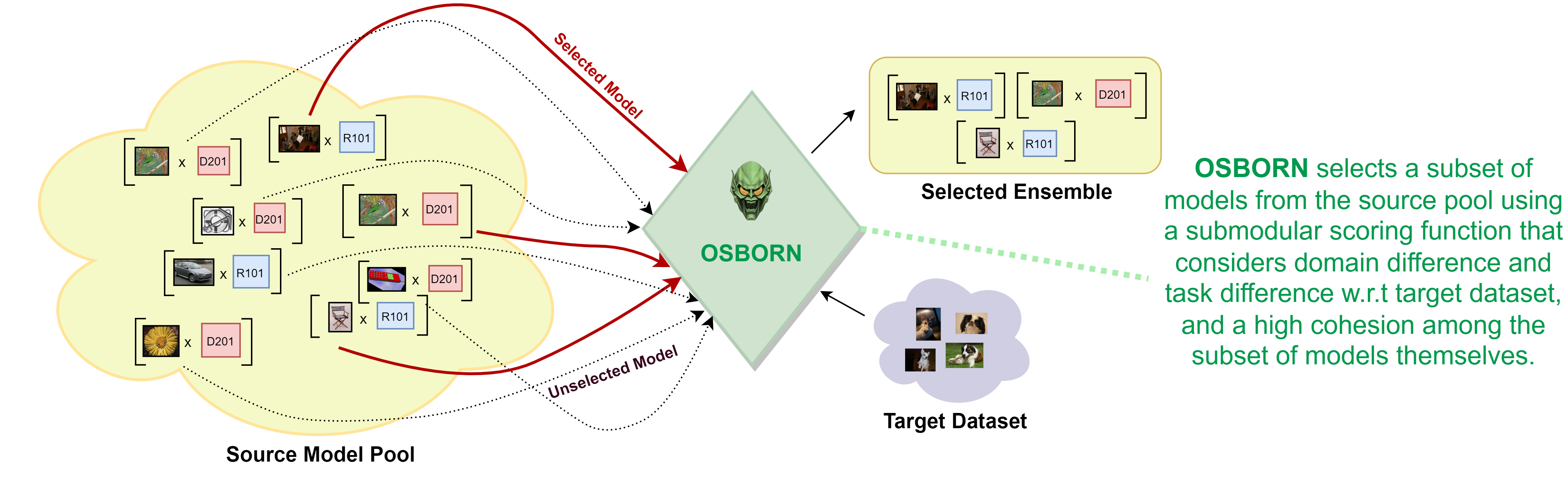}
  \caption{Illustration of the objective and problem setting of our proposed metric. (\textit{Trivia:} OSBORN is also the main antagonist in the Spider-Man movie (2002), hence the emoji.)}
  \label{fig:osborn_overview}
\end{figure*}

In computer vision, transfer learning is a go-to strategy to train Deep Neural Networks (DNNs) on newer domains and datasets across tasks such as image classification \cite{alexnet, resnet}, image segmentation \cite{semseglong, ppn} and object detection \cite{rcnn, rcnnfaster}. This widespread usage is due to the easy availability of a large pool of open-sourced pre-trained models (trained on large-scale datasets such as ImageNet \cite{imagenet, pretrain1}), which, when fine-tuned, achieve faster convergence and better performance than training from scratch. However, every time a user wants to employ transfer learning, the question that has increasingly grown relevant with an increased number of source models is: “Which combination of dataset and architecture should I pick to fine-tune to achieve the best performance on my target dataset?”. To solve this, we need a tool that helps us choose a source or set of source models, which require minimal fine-tuning and achieves maximal performance. 

\textbf{Transferability estimation} (TE) metrics have been proposed in recent years to tackle this problem \cite{nce, leep, logme, otce, gbc}. With these metrics, a particular source model can be selected without conducting expensive fine-tuning of all available source models on the target training set. Most efforts in this direction are, however limited by their capability of selecting only a single source model, thus restricting their use in an ensemble learning setting. There has been only one work so far \cite{eleep} which extends an existing single-source transferability estimation method \cite{leep} to an ensemble setting. While this work showed promising results, it did not consider the similarity between source and target datasets in the latent representation space, or account for the relationships between individual models in the ensemble. This problem space remains nascent at this time, necessitating more efforts to estimate transferability reliably in different conditions.

Ensemble models have been popular for a few decades now in machine learning \cite{adaboost, Breimanbagging, violajones}.
Ensemble models are known to increase task accuracy, decrease overall predictive variance and increase robustness against out-of-distribution data samples \cite{deepensemble1}. Recent efforts have shown the usefulness of ensembles of pre-trained models \cite{wisdom-of-committees2022iclr}, especially considering the widespread availability of pre-trained models in the community \cite{pytorch}. The problem of estimating transferability for a model ensemble from a large source model pool becomes even more relevant in this context.
%With this availability, the question arises "How to select the best models efficiently?". Thus, the problem of estimating transferability of ensemble models is an important problem which can massively benefit the computer vision community.

In this work, we introduce a novel transferability estimation metric specifically designed for ensemble selection called Optimal Transport-based Submodular Transferability metric (OSBORN). As stated earlier, a recent effort in this direction \cite{eleep} showed promising results for such a score, but focused on individual model's performance (via the classifier's outputs) and did not consider the feature (latent representation) space mismatch, or how these models interact with each other in the ensemble. To address this, OSBORN measures the latent space mismatch between the source and the target datasets (domain difference) in addition to the mismatch in the classifier's outputs (task difference). Also, to account for the interaction between models in the ensemble, 
%in \cite{eleep}, no effort has been made towards understanding how these models interact in an ensemble for a given target dataset. To account for it, 
we introduce a novel model cohesion term, which captures the mutual cooperation between models towards forming an ensemble. Cohesion is required to ensure that individual models in an ensemble are in agreement with each other in terms of predictions (and not voting out each other). Thus, in this work, we propose a domain, task and cohesion-aware transferability estimator for ensemble selection from a source pool of multiple models.

% Existing works \cite{eleep} primarily focuses on estimating the transferability of a given candidate ensemble based on individual model's performance. Very little work has been done towards understanding how these models interact in an ensemble for a given target dataset. To account for it, we introduce a novel model cohesion term, which helps to understand the confidence of individual models' predictions of a cooperative ensemble. Cohesion is required to ensure that the individual models in a candidate ensemble are in agreement with each other in terms of predictions. In this work, we propose a cohesion-aware transferability estimator for ensemble selection when there are multiple source models.

Beyond bringing the abovementioned factors into transferability estimation for ensembles, we show that the proposed score can be viewed as a submodular set function \cite{bach2011submodularity}. This allows us to follow a greedy maximization strategy, which is known to provide a high-quality solution for the problem based on well-known theoretical guarantees \cite{nemhauser}. %allow us to simply select the source models by rank-ordering the models according to our score which is utilized to rank the available source models using the greedy maximization algorithm. 
We thus select cohesive and closely related models for a particular target dataset. To evaluate our metric, we conduct extensive experiments using 28 source datasets, 11 target datasets, and 5 model architectures. In downstream tasks, we consider fully-supervised pre-training-based image classification, self-supervised pre-training-based image classification, semantic segmentation as well as domain adaptation. Table \ref{tab:exp_setup} presents an overview of our experiment breadth, as compared to other recent efforts on this problem. In particular, to the best of our knowledge, we are the first to perform transferability estimation of ensembles for image classification and domain adaptation tasks.

To summarize, we make the following contributions: (1) We introduce a novel transferability estimation metric for ensemble selection that considers domain similarity, task similarity and inter-model cohesion in its design; (2) %We propose a model cohesion  term to gauge the agreement of the individual models in the ensemble. (3)
We show that viewing the proposed metric as a submodular set function allows us to use a simple greedy maximization strategy to select a source model ensemble for a given target dataset; %se the submodular nature of our metric to rank the source pool models , which helps select multiple models to form an ensemble (4) 
(3) We study the performance of our metric across a wide range of downstream tasks and model pools;% A comparison of our experimental setup with previous works is given in Tab. $<yet_to_add>$. (5) 
(4) We evaluate the reliability of our metric using different correlation metrics in our studies, and also carry out additional analysis and ablation studies to study its usefulness. We outperform earlier methods by a margin of $58.62\%$, $66.06\%$, and $96.36\%$ in terms of Pearson Correlation Coefficient (PCC), Kendall $\tau$ (KT) \cite{kendaltau} and Weighted Kendall $\tau$ (WKT) \cite{weightedkendall} for the image classification task. \footnote{Project page: \href{https://vimalkb007.github.io/OSBORN/}{https://vimalkb007.github.io/OSBORN/}}

\begin{table}
    \centering
    \scalebox{0.75}{
    \begin{tabular}{l|cccc}%}
    \toprule
    & \multicolumn{4}{c}{\textbf{Single Source TE}} \\ 
     \midrule
     & \thead{Classification} & \thead{Segmentation} & \thead{DA Classification} & \\ 
        \midrule
        \# LEEP \cite{leep} & \checkmark & $\times$ & $\times$\\
        \# LogME \cite{logme} & \checkmark & $\times$ & $\times$\\
        \# OTCE \cite{otce} & $\times$ & $\times$ & \checkmark\\
    \midrule
     & \multicolumn{4}{c}{\textbf{Multi Source TE}} \\ 
     \midrule
        \# MS-LEEP \cite{eleep} & $\times$ & \checkmark & $\times$ \\
        \# Ours & \checkmark & \checkmark & \checkmark\\
    \bottomrule
    \end{tabular}
    }
    \vspace{2mm}
    \caption{Experimental settings studied by different methods in single-source TE and multi-source TE settings (DA: Domain Adaptation). We note the wide range of our experimental settings when compared to earlier work.}
    % \vspace{-6mm}
    
    \label{tab:exp_setup}
\end{table}

\vspace{3mm}
\tocless\section{Related Work}

\noindent \textbf{Transfer Learning:} Over the years, transfer learning has been applied and explored across various fields \cite{bert, transfermedical, transferrl, transferrl2}, as well as across datasets, model architectures, and pre-training strategies \cite{tldatasets, sslbettertransfer, transferssl2}. These efforts have included the study of interesting and practical questions such as which particular layers are more transferable \cite{layerstl} or estimating the correlation between pre-training and fine-tuning performance \cite{plfltl}. Beyond finetuning of source models to target datasets, task transfer methods \cite{taskonomy, taskonomy2} have also studied relationships between visual tasks such as semantic segmentation, depth prediction and vanishing point prediction, or used attribution maps to relate such tasks \cite{deepattr1,deepattr2}. %etc.  The paper \cite{deepattr2} relates different tasks by projecting the pre-trained models to a common space using attribution maps. Similarly, \cite{deepattr1} builds an attribution graph to relate tasks. 
In contrast to the aforementioned methods, the objective of our work is dataset transferability estimation.

% In contrast to these works, our work is based in a different setting which is elaborated in Sec \ref{sec: rel_transfest}.

% This method is essentially much faster than \cite{taskonomy} because it does not require a large set of annotations.

\noindent \textbf{Transferability Estimation Metrics (Single Source):}
\label{sec: rel_transfest_single_source}
As stated earlier, gauging transferability reduces the effort in finding an optimal source model for a particular target dataset because it averts the expensive fine-tuning process. In recent years, significant efforts have been made in this problem space, considering the relevance of this problem to practitioners. The H-Score was proposed \cite{h_score_paper} to measure the usefulness (in terms of discriminativeness) of pre-trained source models for the target task. While this method shows promising results as a pioneer work in this field, it misses considering the scenarios where the source and target data have different distributions. Subsequently, NCE \cite{nce}, and LEEP \cite{leep} developed methods that used the classifier outputs of pre-trained source models when the target dataset is forward-propagated through the model to estimate the log-likelihood of the target dataset. NCE largely focused on estimating transferability in scenarios where the source and target tasks share the same input data (e.g., face recognition and facial attribute classification). %, it is not clear how the metric will perform when the source and target tasks have disjoint input and output spaces (e.g., common objects classification and car variant classification). 
Subsequent methods such as LogME \cite{logme} also showed that likelihood methods might be %They have strict restrictions on the dataset, are 
prone to over-fitting. To tackle this, LogME \cite{logme} estimated the maximum value of label evidence (instead of maximum likelihood) given the feature set extracted by the pre-trained source models.
% LogME \cite{logme} aimed to improve the transferability estimate by maximizing the average maximum (log) evidence of labels given the target sample embeddings. This is achieved by forward-propagating the target dataset through the pre-trained models and obtaining their labels using marginalized likelihood.
%establish the relationship between the features obtained by  It treats each target label as a linear model with Gaussian noise. Then, it optimises the prior distribution parameters to find the average maximum (log) evidence of labels given the target sample embeddings. 
Considering the fact that previous methods largely relied on classifier outputs and their %Although LEEP does not have strong assumptions in the data space, they solely depend on the model outputs for the estimation procedure. Moreover, LEEP is shown to have a 
sub-optimal performance in practical scenarios like cross-domain settings, OTCE \cite{otce} proposed an optimal transport framework to compute domain difference (based on feature space) and task difference (based on label space) to estimate transferability. This method leveraged the source model's latent representations in addition to classifier outputs with no explicit assumptions on  the source and target datasets. All the above works are, however focused on estimating transferability from a single source model to a target dataset. %based on a single source model selection setting. % however, and . does not solely rely on the classifier outputs, and is not prone to overfitting.  %According to this work, both domain difference and task difference hamper transferability. In a single-source model selection setting, OTCE tackles the transferability estimation problem in the most practical way because 

% One strong limitation of NCE is its assumption that the source and target tasks share the same input data although with different label spaces (eg, face recognition and facial attribute classification). 

\noindent \textbf{Transferability Estimation Metrics (Multi-Source Ensembles):}
\label{sec: rel_transfest_multi_source}
Agostinelli et al\cite{eleep} recently proposed the first work on extending transferability estimation to select source model ensembles in \cite{eleep}, specifically focused on semantic segmentation. %For an ensemble, the only metric so far proposed is by \cite{eleep}, an extension to LEEP \cite{leep} designed to tackle the problem of selecting multiple source models in semantic segmentation tasks. 
This work extends LEEP \cite{leep} to ensembles, and shows promising results in the considered settings. Our work builds on this effort in multiple ways: (i) instead of solely relying on classifier outputs for estimating transferability \cite{leep, eleep, nce}, we also consider the domain mismatch in the latent feature representation space; (ii) beyond looking at the individual model's outputs in an ensemble, we also consider the interactions and correlation between the model outputs; (iii) we make no assumptions on the source and target data distributions; and (iv) while \cite{eleep} focused on segmentation, we show our method's results on classification, segmentation and domain adaptation tasks. We also show results on multiple pre-training strategies while previous works \cite{leep, logme, nce, otce} mostly focus on fully-supervised pre-training strategies. Our proposed metric can also be viewed as a submodular function, which allows us to leverage ranking-based greedy optimization strategies to make it efficient in practice. %helps us rank the source models.
%Since they extend LEEP, they also inherit LEEP's drawbacks, which are mentioned above. Our paper focuses on these limitations and proposes a new metric that i) makes no assumptions on the source and target data distributions \cite{nce, leep}, ii) rather than selecting a single source model \cite{nce, leep, logme, otce}, we focus on selecting ensembles, iii) , iv) instead of treating the source models in an ensemble isolatedly \cite{eleep}, we consider the cohesive nature between the source models, v) we also show 

% \begin{table}
%     \centering
%     \scalebox{0.766}{
%     \begin{tabular}{l|cccccc}%}
%     \toprule
%     & \thead{LEEP \\ \cite{leep}} & \thead{LogME \\ \cite{logme}} & \thead{OTCE \\ \cite{otce}} & \thead{MS-LEEP \\ \cite{eleep}} & \thead{Ours \\ (Cls)} & \thead{Ours \\ (Seg)} \\
%     \midrule
%         \# source datasets & 1 & 1 & 4 & 17 & 11 & 10 \\
%         \# pre-training schemes & 1 & 1 & 1 & 2 & 2 & 1 \\
%         \# model architectures & 9 & 10 & 1 & 2 & 4 & 2 \\
%         \midrule
%         \# total source pool & 9 & 10 & 4 & 68 & 88 & 20 \\
%     \bottomrule
%     \end{tabular}
%     }
%     \caption{Overview of our experimental setup and also we compare it with previous works of both single-source model \cite{leep, logme, otce} and multi-source model selection \cite{eleep}. Note: Cls: Classification, Seg: Semantic Segmentation}
%     \vspace{-6mm}
    
%     \label{tab:exp_setup}
% \end{table}

\noindent \textbf{Ensemble Learning.}
Learning ensembles of models has been popular in machine learning to increase overall task performance, decrease prediction variance, prevent over-fitting, and increase out-of-distribution robustness \cite{Breimanbagging, ensembleacc2, ensembleacc3, robustness2}. More recent efforts in training ensembles of neural network models have focused on speeding up their training \cite{ensemblefast, wisdom-of-committees2022iclr}, leveraging a single model's capacity to train multiple subnetworks whose predictions are ensembled to improve robustness \cite{subnetworks}, or studying mixture-of-experts paradigms which bring together thousands of subnetworks for large language models \cite{mixture_of_models}. We clarify that our work focuses rather on selecting model ensembles from a larger source model pool via estimating transferability without explicitly training ensembles themselves. One can view our work as a step before ensemble learning when there is a larger model pool and only few models can be ensembled. As stated in \cite{eleep}, this setting is commonly encountered by a practitioner in the real-world across application domains.

%Since ensembles involve more than one model, training them is computationally expensive. Some works focus on speeding them up, such as the work \cite{ensemblefast} ensembles only the layers close to the output layer. Other work \cite{subnetworks} leverages a single model's capacity to train multiple subnetworks whose predictions are ensembled to improve robustness while keeping the computation time in check.  For larger networks, a study has been conducted on Mixture-of-models \cite{mixture_of_models}, which involves a paradigm of thousands of different subnetworks combinations for large language models.  

%Our work introduces a metric to estimate the transferability of an ensemble of source models and \textit{does not involve any `learning’ paradigm}. In contrast, the mentioned works ‘learn’ strategies to aggregate knowledge from pre-existing expert backbones to solve a task (language modelling, few-shot classification, domain adaptation) at hand.

%-------------------------------------------------------------------------
\vspace{3mm}
\tocless\section{Background and Preliminaries}
\label{sec: bg and prelims}
% In this section, we discuss the necessary notations, definitions and background information required to understand our proposed OSBORN metric.

\noindent\textbf{Notations:} Given $M$ source datasets, we denote the $r^{th}$ source dataset as $D_{s^{r}} {=} \{{(x_{s^{r}}^{i}, y_{s^{r}}^{i})}\}_{i=1}^{n^{r}} \sim P_{s^{r}}(x, y)$ and target dataset as $D_{t} {=} \{{(x_{t}^{i}, y_{t}^{i})}\}_{i=1}^{m} \sim P_{t}(x, y)$ where, $x_{s^{r}}^{i} \in \mathcal{X}_{s^{r}}, x_{t}^{i} \in \mathcal{X}_{t}$, $y_{s^{r}}^{i} \in \mathcal{Y}_{s^{r}}$, and $y_{t}^{i} \in \mathcal{Y}_{t}$. Note that %although image input sizes of the source and target datasets are the same, we do not restrict the domain space $P_{s^{r}}$ $\left( \forall r \leq M \right)$  to be the same as $P_{t}$.
we do not restrict the label spaces $P(\mathcal{Y}_{s^{r}})$ and $P(\mathcal{Y}_{t})$ to span the same category set. We base our study on a domain-agnostic and task-agnostic setting. 

\noindent\textbf{Transferability Estimation for Ensembles:} For every source dataset $D_{s^{r}}$, we assume there exists a pre-trained model on that dataset denoted by $(\theta_{s^{r}}, h_{s^{r}})$ where $\theta$ is the feature extractor, and $h$ is the classifier head. $M$ represents the collection of such source models. As stated earlier, we focus on a multiple source model selection setting (i.e. ensembles) where our metric provides a transferability estimation (TE) score $\alpha^{M_e \rightarrow t}$ for a given subset of models $M_{e}$ from the source pool $M$. When correlated to the accuracy $A^{M_{e} \rightarrow t}$ (i.e. fine-tuned accuracy of the ensemble on the target test set), this TE score provides the reliability of the transferability estimate. Following \cite{eleep}, we calculate the ensemble accuracy by fine-tuning individual models in subset $M_{e}$ (both $\theta$ and $h$) on the target train set and averaging their predictions on the target test set.

%\subsection{Definitions and Notations}
%\label{sec: submodularEnsemble}
% \noindent\textbf{Defintion of Submodularity}
% Definitions of submodularity. Diminishing returns property.
% Why do we need submodularity for transferability estimation. What is the exact motivation behind it.

% \noindent\textbf{Notations Used}
% All the source, target domains, distributions, label spaces notations.

\noindent\textbf{Submodularity in TE for Ensembles.} The main idea of \textbf{TE} involves choosing optimal source models for a given target dataset. Apart from performance \& computation trade-offs, a crucial motivation to select a subset of models is to mitigate risk of negative transfer. 
\begin{wrapfigure}[8]{r}{0.15\textwidth}
\vspace{-13pt}
  \centering
      \includegraphics[height=2.2cm,width=2.56cm]
{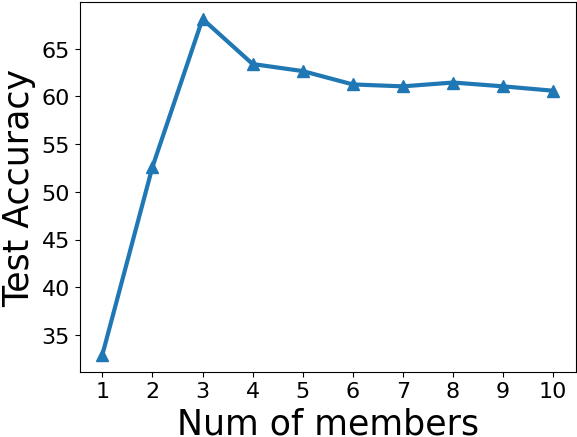}
  \caption{\scriptsize Test accuracies on Caltech101 with varying subsets of models (chosen randomly)}
  \label{fig:caltech101_ensemble}
\end{wrapfigure}
Fig \ref{fig:caltech101_ensemble} herein shows that opting for
all models in the ensemble could
lead to a decrease in overall per-
formance compared to selecting a
smaller set of models. This can
be due to the detrimental impact
of weak or non-transferable models
in the ensemble, highlighting the
importance of carefully combining models to ensure optimal performance. Further, finding an optimal ensemble for a given target dataset requires checking all possible combinations of different source models for a particular ensemble size. This exhaustive process is an NP-hard problem. In this paper, we propose a submodular approach to rank the available models in the source pool according to the performance gain they would yield if added to the subset pool of the ensemble and select the top $k$ models, where $k$ is the required size of the ensemble. While submodular subset selection is popular in different machine learning settings \cite{bach2011submodularity,subopt,wei2015icml}, to the best of our knowledge, this is the first such use for transferability estimation. To this end, we first formally define submodularity below.

% Using the submodularity concept, we can select source models one at a time based on the gain they would yield if added to the subset pool of the ensemble. To this end, we first formally define submodularity.

\begin{defn}
\label{defn:diminishing_returns_submodular_def}
Let $\Omega$ be a set and $\mathcal{P}\left( \Omega \right)$ be the power set of $\Omega$, then a submodular function is a set function $\ensuremath{f} : \mathcal{P}\left( \Omega \right) \to \mathbb{R}$. The submodular function follows the property of diminishing returns, i.e. adding a new element to a smaller set produces a larger increase in $\ensuremath{f}$ compared to a larger set. Mathematically, if for all $\ensuremath{X}, \ensuremath{Y} \subseteq \Omega$, where $\ensuremath{X} \subseteq \ensuremath{Y}$ and for all $\ensuremath{v} \in \Omega\setminus\ensuremath{Y}$, the property follows:
\vspace{-1.5mm}
    \begin{equation}
     \ensuremath{f}(\ensuremath{X} + \ensuremath{v}) - \ensuremath{f}(\ensuremath{X}) \geq \ensuremath{f}(\ensuremath{Y} + \ensuremath{v}) - \ensuremath{f}(\ensuremath{Y})  
     \label{eq:subdiminishingreturns}
    \end{equation}
\end{defn}

A key benefit of posing a problem as one of submodular subset selection is that a greedy approach can be leveraged to efficiently identify a solution of required subset size that is reasonably close to the optimal solution. Nemhauser \cite{nemhauser} showed that the quality of the subset chosen greedily cannot be worse than $1 - e^{-1}$ of the optimal value. This makes submodularity an attractive approach for usage in the field of TE for ensembles as we can rank the models in the source pool and select an ensemble of desired size. Further details on how to greedily select the models are discussed later in this paper.

% This makes submodularity an attractive approach for usage in TE, especially in an ensemble setting. Notably, we are the first to use the concept of submodularity in the field of TE for ensembles.
 % Submodularity works by selecting models one at a time based on the gain they would yield if added to the subset pool of the ensemble.

 % This makes submodularity an attractive and ideal approach for usage in the field of \textbf{TE}, especially in an ensemble setting where we choose models one at a time from the source models pool. 

\noindent\textbf{Evaluation Criteria.}
% \begin{enumerate}
%     \item performance reliability of transferability estimates via Pearson Correlation, KT, WKT. Adding necessary references here. 
% \end{enumerate}
As stated earlier, the reliability of a \textbf{TE} method is obtained by measuring the correlation between $\alpha^{M_{e} \rightarrow t}$ and  $A^{M_{e} \rightarrow t}$. Previous works \cite{logme, leep, eleep, otce, nce} measure this correlation using different techniques such as Pearson Correlation Coefficient (PCC), Kendall $\tau$ (KT) \cite{kendaltau} and Weighted Kendall $\tau$ (WKT) \cite{weightedkendall}. We report results for all these correlation measures to be comprehensive in our analysis.

\tocless\section{OSBORN: Transferability Estimation Metric for Model Ensemble Selection}
\label{sec:prop_metric}

%\subsection{OSBORN}
% \begin{enumerate}
%     \item define the metric
%     \item define $W_D$ domain difference, and it's the importance
%     \item define $W_T$ task difference, and it's the importance
%     \item define $W_C$ Cohesiveness Term. Introducing what the Cohesiveness term here is. Why model cohesiveness is required and how it will help the existing metrics. Including model cohesiveness term to existing methods should be added to the appendix table and reference here. The entire big table we have in Sheet 1
%     \item give a reference to appendix on the study of $\lambda$'s
%     \item in the end, one or two lines about OTCE, and will add the OTCE vs OSBORN comparison table for ensemble here.
% \end{enumerate}

In order to design a reliable transferability estimation approach for model ensembles, we propose the Optimal Transport-based Submodular Transferability metric (OSBORN), which considers three factors: domain difference, task difference, and inter-model cohesion. Inspired by earlier efforts on single-source transferability estimation~\cite{otce}, we consider both classifier output and distance in the latent representation space in our approach. Besides, since our focus is on model ensembles, we consider inter-model relationships in this metric. We now describe each of these quantities.

% We define our metric OSBORN for estimating the transferability of a candidate ensemble in a domain-agnostic and task-agnostic setting as:

\vspace{3pt}
\noindent \underline{\textit{Minimize Domain Difference ($W_D$).}} 
In order to minimize the latent space mismatch between the source and target datasets, similar to \cite{otce}, we choose Wasserstein distance and Optimal Transport (OT) to compute this mismatch owing to its advantages in capturing the geometries of underlying data. Mathematically, the p-Wasserstein distance is given as follows: 
\vspace{-1.5mm}
\begin{equation}
    W_{p}\left( \beta, \gamma \right) = \left(\inf_{\pi \in \Pi(\beta, \gamma)} \int D(x, z)^{p}d\pi(x,z)\right)^{1/p}
    \label{eq:p-wass}
\end{equation}
\noindent where, $p \geq 1$, $\beta, \gamma$ are continuous or discrete random variables in a complete and separable space ${S}$, $D(.,.) :$  $S \times S \rightarrow \mathbb{R}^{+}$ is a distance or a cost function between two points $x$ and $z$, $\pi(\beta, \gamma)$ is the coupling matrix which can also be understood as the joint probability distributions with marginals $\beta$ and $\gamma$. 
In particular, in this work, we use the 1-Wasserstein distance, also called the Earth Mover Distance, to calculate the domain difference between source and target latents as:
\vspace{-1.5mm}
\begin{equation}
    W_{D}\left( \theta_{s}, x_{t} \right) = \sum_{i,j=1}^{m,n} || \theta_{s}(x_{s}^{i}) - \theta_{s}(x_{t}^{j})||_{2}^{2}\pi_{ij}^{*},
    \label{eq:WD}
\end{equation}
\noindent where $||\cdot-\cdot||_{2}^{2}$ is the distance or cost metric, $\pi^{*}$ is the optimal coupling matrix of size $m\times n$ obtained by solving the optimal transport (OT) problem using the Sinkhorn algorithm \cite{sinkhorn, otce}. Note that $\theta_{s}(.)$ is the feature extractor belonging to the source model. Intuitively, if the latent space of the source dataset is closely aligned with that of the target dataset, it is easier for the model to transfer.

% It is the domain difference between the source and the target dataset's latent representations calculated using 1-Wasserstein distance (Eq. \ref{eq:p-wass}). The source and target latents are obtained by forward passing the training set images of the source and target datasets through a source model's feature extractor $\theta$. To efficiently transfer a source model to a particular target dataset, the representational space of the source dataset is expected to be close to the target dataset. So ideally, $W_D$ is lower for easily transferable models.

\vspace{3pt}
\noindent \underline{\textit{Minimize Task Difference ($W_T$).}} In order to measure the difference between a source task and the given target task, we use the mismatch between the model/classifier's outputs for source and target data forward-propagated through the source model. We use the conditional entropy (CE) of the predicted labels $\hat{y_{t}} \in \mathcal{Y}_s$ of the target dataset samples given their ground truth labels $y_{t} \in \mathcal{Y}_t$. The predicted labels are obtained by forward-propagating the target samples $x_t$ through the corresponding source model $\theta_{s}$. Let $\hat{Y_t}$ be a random variable that takes values in the range of $\mathcal{Y}_{s}$; and $Y_t$ be a random variable that takes values in the range of $\mathcal{Y}_{t}$, then $W_T$ can be calculated as:
\vspace{-2mm}

\begin{align}
    &\begin{aligned}
      W_{T}\left( \theta_{s}, x_{t} \right) &= H(\hat{Y_t}|Y_t) \\ 
      \qquad &= -\sum_{\hat{y_{t}} \in \mathcal{Y}_{s}} \sum_{y_{t} \in \mathcal{Y}_{t}} \hat{P}(\hat{y_{t}}, y_{t})\log\frac{\hat{P}(\hat{y_{t}}, y_{t})}{\hat{P}(y_{t})}
    \end{aligned}
    \label{eq:WT}
\end{align} 

\noindent where $\hat{P}(\hat{y_{t}}, y_{t})$ is the joint distribution of predicted and ground truth target labels and $\hat{P}(y_{t})$ is the marginal distribution of the ground truth labels. These quantities can be easily computed using the optimal coupling matrix (obtained in Eqn \ref{eq:WD}) as follows: 

\begin{equation}
    \hat{P}(\hat{y_{t}}, y_{t}) = \sum_{i,j:\hat{y_{t}^{i}}=\hat{y_{t}}, y_{t}^{j}=y_{t}} \pi_{ij}^{*},
    \label{eq:joint_prob}
\end{equation}

\noindent The marginal distribution can be obtained from the joint distribution as follows:
\vspace{-2mm}
\begin{equation}
    \hat{P}(y_{t}) = \sum_{\hat{y_{t}} \in \mathcal{Y}_{s}} \hat{P}(\hat{y_{t}}, y_{t}),
    \label{eq:marginal_prob}
    \vspace{-1mm}
\end{equation}

\noindent Intuitively, similar tasks will result in a low $W_T$ value. Using $W_T$ i.e CE alone represents empirical transferability according to \cite{nce}. However, in \cite{otce}, it is experimentally shown that using only CE is insufficient in a domain-agnostic setting, which motivates us to combine this with $W_{D}$ to account for feature representation space mismatch.

\vspace{3pt}
\noindent \underline{\textit{Minimize Model Disagreement (Cohesiveness $W_C$).}} For an ensemble, it is important that the individual models reinforce the predictions of each other and have less disagreement amongst themselves to have overall good performance. To understand the cohesiveness of an ensemble, we use Conditional Entropy to capture the amount of disagreement between models in the subset of models $M_e$. Mathematically, we represent $W_C$ as:
\vspace{-2mm}
\begin{equation}
W_C\left( M_{e}, x_{t}\right) = \sum_{m_i, m_j \in M_{e}} H(m_i({x}_{t})|m_j({x}_{t}))
\label{eq:model_cohesive}
\vspace{-1mm}
\end{equation}

\noindent Intuitively, we want a high cohesiveness and less disagreement among the models to reinforce the ensemble's predictive belief, i.e. a low $W_C$ value, and to avoid scenarios where models vote out each other's predictions. 

Bringing the quantities together, we define OSBORN for a subset of models $M_e$ of our source pool $M$ as follows. Our metric collectively accounts for domain difference, task difference and model cohesion. Ref. Fig. \ref{fig:method_diag2} for the overview.
\vspace{-2mm}
\begin{align}
    &\begin{aligned}
      \text{OSBORN} &= \sum_{m_i \in M_e}[ W_D( m_i, x_{t}) +  W_T( m_i, x_{t})] +\\
      &\qquad W_C\left( M_{e}, x_{t}\right)
      \vspace{-1mm}
      \label{eq:osborn_func}    
    \end{aligned}
\end{align} 

\begin{figure}[]
  \raggedleft
  \includegraphics[width=1.1\linewidth]
    {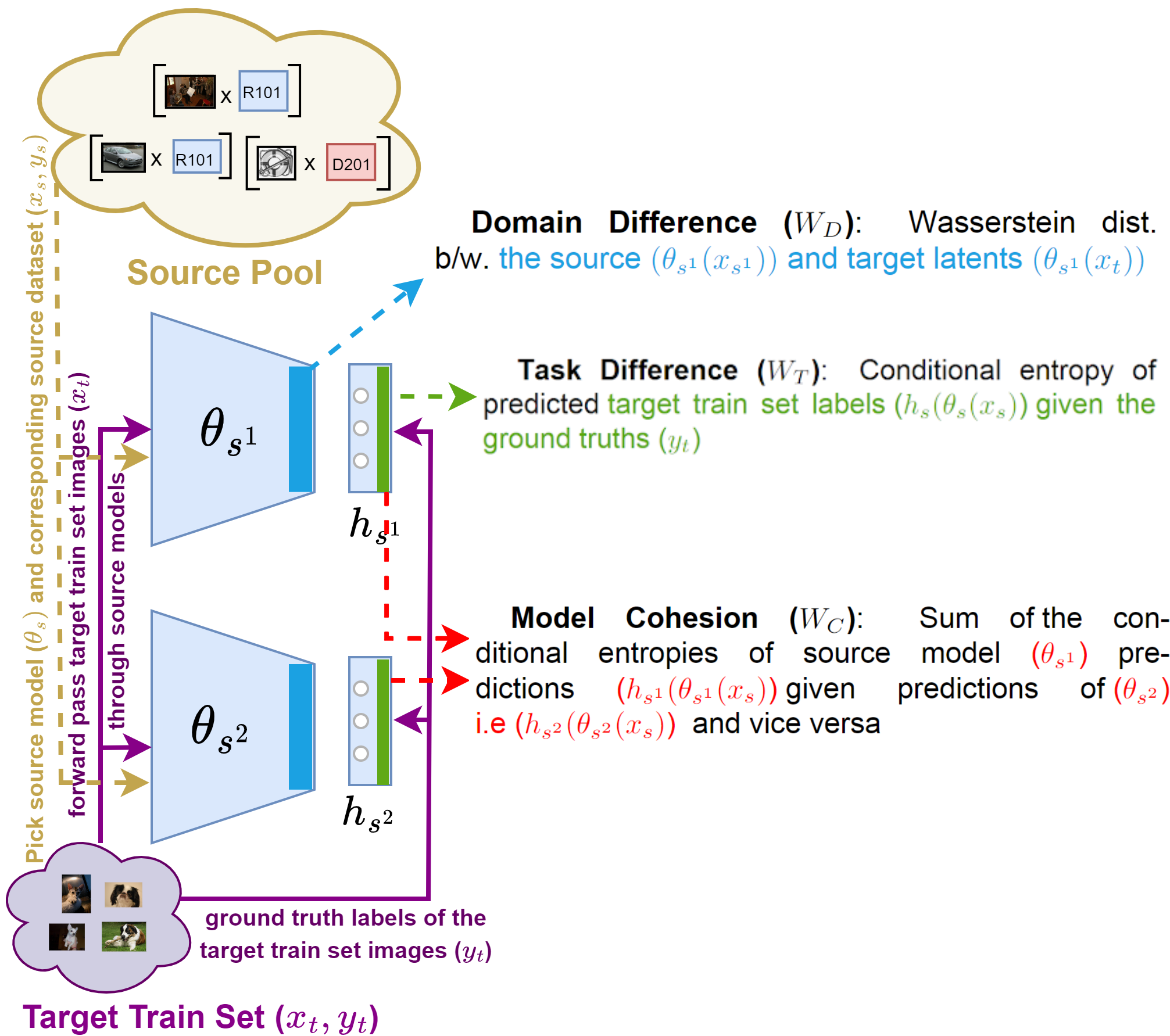}
  \vspace{-7pt}
  \caption{Overview of our method for estimating the transferability for ensembles.}
  \label{fig:method_diag2}
  \vspace{-5mm}
\end{figure}

\noindent A model ensemble that obtains a low OSBORN score will have better transferability to a target dataset. Our experiments show that a simple combination of these three quantities (with no weighting co-efficients) outperforms existing methods in all our experiments. In our ablation studies and analysis, we study the contribution of each OSBORN component as well as the effect of weighting each component differently. %and how it varies with datasets. Also, we compared our metric with \cite{otce}. Further details about these experiments are provided in the supplementary material. 

\vspace{4pt}
\noindent \textbf{Submodular Subset Selection in OSBORN.}
% \begin{enumerate}
%     \item introduces the submodular nature of OSBORN via negation
%     \item theorem the proposed solution is submodular
%     \item proof of the theorem we can include in the appendix
% \end{enumerate}
As stated earlier, we show that the proposed OSBORN metric translates to a submodular optimization problem, which allows us to rank and pick models efficiently from the source pool. 
% Per our metric (Eq. \ref{eq:osborn_func}), we want to observe a low OSBORN score for an ensemble to have better transferability.
While the aforementioned quantities were written from a \textit{minimization} perspective (for clarity and ease of understanding), to pose this as a submodular \textit{maximization} problem, we consider the corresponding scoring function to be maximized as: 
%Generally, a submodular function is maximized to select the most informative subset of data points with minimum cardinality and achieve maximum accuracy (for e.g. data subset selection in Active Learning \cite{wei2015icml}). 
% In Active Learning \cite{subact}, a submodular function is maximized to select the most informative subset of data points with minimum cardinality and achieve maximum accuracy. 
%Since we want to minimize our metric ( Eq. \ref{eq:osborn_func}), we choose to see it in a submodular way of maximizing its negative. The scoring function of the subset $M_e$ for the target dataset $D_t$ can be written as:
\vspace{-2mm}
  \begin{align}
      &\begin{aligned}
         f\left( M_e \right) &= -\sum_{m_i \in M_e}[ W_D( m_i, x_{t}) +  W_T( m_i, x_{t})] - \\
      &\qquad W_C\left( M_{e}, x_{t}\right)
      \end{aligned} 
      \label{eq:scoring_func}
  \end{align}

The value of our set function is a transferability estimate designed such that it is highly correlated
to the fine-tune accuracy (see Table \ref{table:fully_supervised} \& \ref{table:self_supervised}), thus enabling us to select models without expensive fine-tuning.

\begin{thm}
\vspace{-1mm}
The scoring function $f\left( X \right)$, as defined in Equation \ref{eq:scoring_func}, is a submodular function.
\end{thm}

\begin{proof}
Let $X_1$ and $X_2$ be two sets such that $X_1 \subseteq X_2 \subseteq M$. If we consider an unselected model instance $v \in M\symbol{92}X_2$. The gain in the score is obtained by appending $v$ to the set $X_1$, and this is calculated as:
\vspace{-2mm}
  \begin{align}
      &\begin{aligned}
          f\left( X_1 \cup v \right) - f\left( X_1 \right) &=  - \left[ W_D\left( v, x_{t} \right) +  W_T\left( v, x_{t} \right) \right]\\
      &\qquad -  \sum_{m_i \in X_1} H\left( m_i\left(x_{t} \right) \mid v\left(x_t \right)  \right)\\
      &\qquad - \sum_{m_j \in X_1} H\left( v\left(x_t \right) \mid m_j\left(x_t \right)  \right) 
      \end{aligned} 
      \vspace{-10pt}
      \label{eq:X_1_scoring_func_gain}
  \end{align} 

\noindent Similarly, the gain obtained by set $X_2$ is given by:
\vspace{-2mm}
  \begin{align}
      &\begin{aligned}
          f\left( X_2 \cup v \right) - f\left( X_2 \right) &=  - \left[  W_D\left( v, x_t \right) +  W_T\left( v, x_t \right) \right]\\
      &\qquad -  \sum_{m_i \in X_2} H\left( m_i\left(x_t \right) \mid v\left(x_t \right)  \right) \\
      &\qquad - \sum_{m_j \in X_2} H\left( v\left(x_t \right) \mid m_j\left(x_t \right)  \right) 
      \end{aligned} 
      \label{eq:X_2_scoring_func_gain}
  \end{align} 
  \vspace{-10pt}

\noindent As we have $X_1 \subseteq X_2$, the number of terms in the summation of Equation \ref{eq:X_2_scoring_func_gain} will be greater than or equal to that of Equation \ref{eq:X_1_scoring_func_gain}. Since entropy is always a non-negative value, we can say that 
\vspace{-1mm}
\begin{align*}
    &\begin{aligned}
        -\sum_{m_i \in X_1} H\left( m_i\left(x_t \right) \mid v\left(x_t \right)  \right) - \sum_{m_j \in X_1} H\left( v\left(x_t \right) \mid m_j\left(x_t \right)  \right) &\geq \\
         - \sum_{m_i \in X_2} H\left( m_i\left(x_t \right) \mid v\left(x_t \right)  \right)
         - \sum_{m_j \in X_2} H\left( v\left(x_t \right) \mid m_j\left(x_t \right)  \right)  
         \vspace{-20pt}
    \end{aligned}
\end{align*}
This implies that
\vspace{-2mm}
\begin{align}
    f\left( X_1 \cup v \right) - f\left( X_1 \right) &\geq  
      f\left( X_2 \cup v \right) - f\left( X_2 \right)\label{eq:inequality_proof}
\end{align}

We can see that Equation \ref {eq:inequality_proof} satisfies the condition in Definition \ref{defn:diminishing_returns_submodular_def}.
This completes the proof. 
\end{proof}

\vspace{3pt}
\noindent \textbf{Submodular Optimization using Greedy Maximization.}
% \begin{enumerate}
%     \item We define the algorithm for the optimization
%     \item Nemhauser and Lovasz theorem about the claim on 63\% of optimal solution.
%     \item reference to the above theorem's proof. Maybe a possible footnote
%     \item cost performance of the algorithm we defined. Include the graph of models selection for datasets.
%     \item greedy algorithm's efficiency to find Near Optimal Solution
%     \item add the table to show the results in terms of accuracy of optimal ensemble and greedy algorithm selected ensemble further to reinforce $\geq 63\%$ of optimal solution claim.
% \end{enumerate}
 Since our set function $\ensuremath{f}(\ensuremath{M_e})$ (mentioned in Eq. \ref{eq:scoring_func}) is submodular, it exhibits monotonicity, i.e. the set with maximum gain is always the entire source pool $M$. However, since we want to select a subset of models i.e. ensemble set from the source pool $M$, we impose a cardinality constraint. Formally, we aim to select the set $M_e$ of size at most k that maximizes the gain:
 
 % As we have shown that the set function $\ensuremath{f}(\ensuremath{X})$ is submodular, we aim to select the set $X$ satisfying the following objective:
\vspace{-3.5mm}
 \begin{equation}
     \max_{\ensuremath{M_e}:|\ensuremath{M_e}| = \ensuremath{k}} \ensuremath{f}(\ensuremath{M_e})
 \end{equation}

This problem is however NP-hard, but we use the greedy maximization strategy to find a near-optimal set of models $\ensuremath{M_e}$ for the target dataset. In practice, we pre-calculate pair-wise domain difference $W_D$ and task difference $W_T$ between each source and target datasets. Then, we calculate the model cohesion term $W_C$ for adding each model $m_i$ to the set of already selected models $M_e$. Using these three quantities pertaining to $m_i$, we calculate the gain achieved by adding it to the set $M_e$ as $f\left( M_e \cup m_i \right) - f\left( M_e \right)$ and greedily pick the model with the highest gain and add it to the set $M_e$. We continue this iteration till we achieve the ensemble set size of $k$. Once the target samples are forward-propagated through the source models, the quantities in our metric can be computed independently for each source model, thus making our overall computations parallelizable. 

 % for selecting each model $m_i$, we calculate the gain of it using $f\left( M_e \cup m_i \right) - f\left( M_e \right)$, wherein we calculate the model cohesion term $W_C$ on the fly and greedily pick the model with the highest gain and add it to the set $M_e$. We continue this iteration till we achieve the ensemble set size of $k$. 
 % Please refer to Algorithm \ref{algo:greedy} for details.
 
 %\noindent \textbf{Greedy Algorithm guarantee.} 
 %Using the greedy maximization algorithm, we rank order the models in the source pool and choose the top $k$ models to form an ensemble. 
 Considering $\ensuremath{M_{e}^{\ast}}$ as the optimal ensemble set, it is well-known from \cite{nemhauser} that such a greedy approach has a performance guarantee of  at least $63\%$ of the optimal ensemble set,  i.e.

% For selecting each of the models $(m_i)$,

% \begin{algorithm}
% \scriptsize
% \caption{Greedy Maximization Algorithm for Submodular Function}\label{alg:cap}
% \begin{algorithmic}
% \Require The submodular function $f$. Calculated pairwise Wasserstein distance between Individual source model latents $\theta^{m}(x_{s}^{i})$ and target latents $\theta^{m}(x_{t}^{i})$ as $W_D(\theta^{m}(x_{s}), \theta^{m}(x_{t}))$. Task Differences $W_T(h_{m}(x_{s}), h_{m}(x_{t}))$. Cohesiveness $W_C(h_{m_i}(x_{t}), h_{m_j}(x_{t}))$.\\

% \Ensure $k \leq M$
% \State $k \gets Ensemble Size$

% \State $EnsembleSet \gets \{ \}$
% \For{$i \gets 1$ to $k$}
%     \State $InterimList \gets \{ \}$
%     \For{$m \gets 1$ to $M$}
%         \If{$m \notin EnsembleSet$}
%             \State $InterimList \gets f(m, Ensemble Set)$
%         \EndIf
%     \EndFor
%     \State $X \gets \max(InterimList)$
%     \State $EnsembleSet \gets EnsembleSet + X$
% \EndFor
% \end{algorithmic}
% \label{algo:greedy}
% \end{algorithm}

% Let $\ensuremath{X^{\ast}}$ be the optimal ensemble set,  then the ensemble set found by the greedy algorithm \cite{nemhauser} has the guarantee of  $\geq 63\%$ of optimal ensemble set. i.e.
\vspace{-3mm}
\begin{equation}
    \ensuremath{f}\left(\ensuremath{M_{e}}\right) \geq \left( 1 - \frac{1}{e}\right)\ensuremath{f}\left(\ensuremath{M_{e}^{\ast}}\right)
\end{equation}

In practice, we observe that we see that the avg. accuracy of the ensemble selected by greedy ($76.315\%$) in a fully-supervised setting is, $95.56\%$ of the avg. accuracy of the optimal ensemble($79.857\%$). Similarly for self-supervised setting, the avg. accuracy of the ensemble selected by greedy ($79.857\%$) is, $93.50\%$ of the avg. accuracy of the optimal ensemble($84.962\%$), as shown in Table \ref{table:greedy_optimal_accuracy}. More details on the experiments are presented in the next section.

%\ref{sec:exptsandresults}. % for further details. 

% We notice the cost performance of the submodular function while selecting models one at a time increases as the subset size increases. This is due to calculating each model's gain and choosing the model with the highest gain. We show the cost performance of models selected for the Caltech101 dataset in Fig. \ref{fig:tr1}.

\begin{table}[h]
\centering
\scalebox{0.8}{
\begin{tabular}{ccc}
\toprule
       & \multicolumn{2}{|c}{\textbf{Ensemble   Accuracy (Fully Supervised)}} \\ \cmidrule{2-3}
\multirow{-2}{*}{\textbf{Target Dataset}}         & \multicolumn{1}{|c|}{\textbf{Greedy}} & \textbf{Optimal}       \\ \midrule
Oxford102Flowers & \multicolumn{1}{|c|}{90.720}& {91.697}\\ 
Caltech101       & \multicolumn{1}{|c|}{68.533} & {75.333}\\ 
StanfordCars     & \multicolumn{1}{|c|}{69.692}& {72.540} \\ \midrule
\textbf{Average} & \multicolumn{1}{|c|}{76.315}& {79.857}\\ \midrule
 & \multicolumn{2}{c}{\textbf{Ensemble   Accuracy (Self Supervised)}} \\ \midrule
Oxford102Flowers & \multicolumn{1}{|c|}{86.935}& {95.604}\\ 
Caltech101       & \multicolumn{1}{|c|}{88.800}  & {90.000}    \\ 
StanfordCars     & \multicolumn{1}{|c|}{62.604}& {69.282}\\ \midrule
\textbf{Average} & \multicolumn{1}{|c|}{79.446}& {84.962}\\ \bottomrule
\end{tabular}
}
\vspace{2mm}
\caption{Comparison of the target test set accuracies achieved by fine-tuned ensembles selected using the greedy optimization of OSBORN vs the optimal ensembles. We clearly observe that our approach empirically gives significantly stronger performances than the theoretical guarantee.}% Note: The optimal ensemble is an ensemble with the highest accuracy.}
\vspace{-2mm}
\label{table:greedy_optimal_accuracy}
\end{table}

%\vspace{3pt}
%\noindent \textbf{Time Complexity.}
% \begin{enumerate}
%     \item OTCE Time Complexity
%     \item submodular function complexity
%     \item Algorithm's complexity
% \end{enumerate}
%To compute OSBORN for an ensemble, we first extract the source and target latents using the source model's feature extractor $\theta(.)$. The time complexity for extracting the latents will be $O(NF)$, where $N$ denotes the number of images, and $F$ denotes the complexity of extracting one latent. Since the latent extraction is independent for each model-dataset pair, they can be calculated in parallel. For the OT part, we use Sinkhorn's algorithm \cite{sinkhorn}, which has the complexity bound of $\tilde O(\frac{N^{2}}{\epsilon^{2}})$, where $\tilde O$ hides polylogarithmic factors $(\ln)^{C}$, $C > 0$ denotes the cost matrix \cite{otcomplexity}. The complexity of calculating the Wasserstein distance between an individual source and a target domain is independent of the source and target pair, which again gives us the advantage of pre-calculating these distances in parallel. The total time complexity of OSBORN results in the  summation of extracting latents and the OT complexity bound.

%-------------------------------------------------------------------------

\tocless\section{Experiments and Results}
\label{sec:exptsandresults}
% \subsection{Experimental Setup}
% \begin{enumerate}
%     \item need to clearly define that we do not include target dataset models while calcuating or selecting the candidate ensemble
%     \item we can possibly mention that we are first one to show ensemble results to most commonly used setup of image classification
%     \item why are we only calculating ensembles of size 3. A graph needs to be included that shows the accuracy vs ensemble sizes, and show that ideal ensemble size for our experimental setup is 3. Make sure to include the ensemble size 1 as we didn't do it in CVPR. Which is also in accordance with MS-LEEP paper.
%     \item Explain combinations and transferbility estimation that requires. 
%     \item \textbf{Classification Datasets} details
%     \item \textbf{Segmentation Datasets} details
%     \item \textbf{Model Architectures - Fully Supervised Learning} 
%     \item \textbf{Model Architectures - Self Supervised}
%     \item \textbf{Model Architectures - Semantic Segmentation}
%     \item we need to mention about the learning rate details. One of the reviewer's mentioned we used constant LR for all the models, which was not the case. We either need to tell them what were the LR for each dataset-model pair (or) we need to write it bit differently about the LR being not constant in the supplementary.
% \end{enumerate}

\noindent \textbf{Experimental Setup.}
We follow the same experimental setup as the previous work on source model ensemble selection \cite{eleep} to evaluate our transferability metric in the multiple source model setting. Given a total of $M$ models in the source pool, our objective is to select an ensemble model by choosing $k$ models from the source pool. %To calculate the correlation values, we consider all the possible ways to choose $k$ models from the source pool, i.e. $\Comb{M}{k}$ and calculate transferability using OSBORN. 
We follow \cite{eleep} in setting $k$ to 3 for fairness of comparison. %and 
%For a fair comparison with the benchmarks and to make it computationally less expensive, we set $k$ to 3, the same as in \cite{eleep}, to keep it consistent.
% \vspace{-2mm}
We also conducted a study to evaluate this on the Oxford-IIIT Pets dataset, and found that maximum accuracy is gained for an ensemble of size 3 (see Fig \ref{fig:ensemblesize}), which further reinforces our choice for conducting experiments.
% We also check for different ensemble sizes in the range of 1 to 5 and their corresponding accuracy, as shown in Fig. \ref{fig:ensemblesize}  for Oxford IIIT Pets \cite{oxfordiiitpets} dataset. It is noticed that maximum accuracy is gained when three members are in the ensemble, which is our choice for our experiments in the main paper. 
% \vspace{-4mm}

\noindent \textbf{Classification Datasets.} For the classification tasks, we consider 11 widely-used datasets including  CIFAR-10 \cite{cifar100}, CIFAR-100 \cite{cifar100}, Caltech-101 \cite{caltech101}, Stanford Cars \cite{stanfordcars}, Oxford 102 Flowers \cite{oxford102flowers}, Oxford-IIIT Pets \cite{oxfordiiitpets}, Imagenette \cite{imagenette}, CUB200 \cite{cub200}, FashionMNIST \cite{fashionmnist}, SVHN \cite{svhn}, Stanford Dogs \cite{stanforddogs}. These datasets are popularly used in many transfer learning tasks. Out of these 11 datasets, we set Caltech-101 \cite{caltech101}, Stanford Cars \cite{stanfordcars}, Oxford 102 Flowers \cite{oxford102flowers}, Oxford-IIIT Pets \cite{oxfordiiitpets}, Stanford Dogs \cite{stanforddogs} as our target datasets and estimate transferability using OSBORN.

\noindent \textbf{Model Architectures (Fully-supervised).} For this setting, we consider 2 source model architectures ResNet-101 \cite{resnet} and DenseNet-201 \cite{densenet}, keeping in mind the model diversity and capacity. We take these models from the open-sourced PyTorch Library \cite{pytorch}. Initially, both the models are initialized with the fully-supervised ImageNet weights \cite{imagenet}, and then we train them on the 11 classification datasets to prepare our source model pool.

\noindent \textbf{Model Architectures (Self-supervised).} For this setting, we consider ResNet-50 \cite{resnet} as our source model architecture but initialize it with weights obtained from two self-supervised pre-training strategies, namely BYOL \cite{byol} and MoCov2 \cite{mocov2}. We have two variants of ResNet-50 models to produce enough diversity. And as done in the previous case, we train these two models on the 11 classification datasets to prepare our source model pool. We use multiple pre-trained SSL models to build our pool. However, finetuning is done in a fully-supervised fashion.  Our motivation here was to study if OSBORN can estimate transferability reliably across multiple pre-training settings.

\noindent \textbf{Training Setup for Source Models (Classification Tasks).} For all classification tasks, we train the source models using a cross-entropy loss and optimize it using Stochastic Gradient Descent (SGD) with momentum. Given these details, the most important hyperparameters are learning rate, batch size and weight decay. We train the models with a grid search of learning rate in ($1e$$-$$1$, $1e$$-$$2$, $1e$$-$$3$, $1e$$-$$4$), batch size in ($32, 64, 128$), and weight decay in ($1e$$-$$3$, $1e$$-$$4$, $1e$$-$$5$, $1e$$-$$6$, $0$) to pick the best hyperparameters. All our experiments are written in PyTorch and are conducted on a single Tesla V-100 GPU. For the fully-supervised pre-trained setting, we initialize the models with ImageNet weights. In the case of a self-supervised pre-trained setting, we initialize the models using BYOL or MoCov2 (on ImageNet) weights. For our experiments on the multi-domain DomainNet dataset, we initialize our models with ImageNet weights.

\noindent \textbf{Training Setup for Source Models (Semantic Segmentation Tasks).} We train the source models using a pixel-wise cross-entropy loss and optimize it using Stochastic Gradient Descent (SGD) with momentum. The most important hyperparameters herein are learning rate, batch size and weight decay. We train the models with a grid search of learning rate in ($1e$$-$$1$, $1e$$-$$2$, $1e$$-$$3$, $1e$$-$$4$), batch size in (${32, 64, 128}$), and weight decay in ($1e$$-$$3$, $1e$$-$$4$, $1e$$-$$5$, $1e$$-$$6$, $0$), and pick the best hyperparameters. All these experiments are also written in PyTorch and conducted on a single Tesla V-100 GPU. We initialize source models using the COCO pre-trained weights.

\noindent \textbf{Implementation of Source Models and Baselines.} We use open-source models available via the PyTorch Library for classification and semantic segmentation tasks. We use the PyTorch Lightning Library to obtain model weights for a self-supervised pre-training setting. We use the code released by the respective papers for calculating OTCE \cite{otce}, MS-LEEP, E-LEEP, IoU-EEP and SoftIoU-EEP \cite{eleep} scores.

\begin{wrapfigure}[13]{r}{0.28\textwidth}
% \vspace{-4pt}
  \centering
    \includegraphics[width=1\linewidth]
{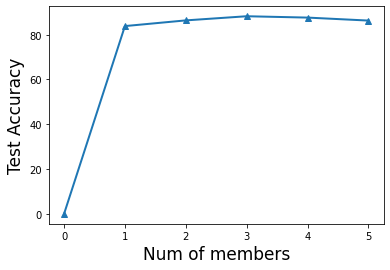}
  \caption{Test accuracy on the Oxford-IIIT Pets dataset compared to the ensemble size. We observed a similar trend across other datasets as well.}
  \label{fig:ensemblesize}
\end{wrapfigure}

\noindent \textbf{Evaluating Ensemble Performance.} We follow the protocol in \cite{eleep} for computing ground truth accuracies of ensembles. We finetune (both feature extractor and classifier of) all the source models present in the ensemble using the target training set. Then, we individually make predictions using the source models on the target test set and average them to get the final ensemble prediction. We note that no target-trained models are in the source pool. We compare this final prediction with the ground-truth label and calculate the classification accuracy. Note that we need to fine-tune all source models only once and can re-use their predictions on the test set across all ensemble combinations. As stated earlier, we report Pearson Correlation Coefficient (PCC), Kendall $\tau$ (KT) and Weighted Kendall $\tau$ (WKT) in our results.

%\noindent \textbf{Important Notice.}
%To reproduce the results of our baselines MS-LEEP and E-LEEP, we contacted the authors of the paper \cite{eleep} for their custom model architectures and  pre-trained weights; and code scripts for calculating the metrics. We thank the authors for providing us with the code for metric calculation (see supplementary material). The authors, however, did not share their custom model architectures due to their constraints. Hence, we used off-the-shelf models for all the experiments available at open-source platforms like PyTorch \cite{pytorch} closest to their work to ensure our work is more accessible, relevant and reproducible to the computer vision community.
\begin{table*}[h]
\centering
\footnotesize
\scalebox{0.9}{
\begin{tabular}{c|ccc|ccc|ccc}
\toprule
 &
  \multicolumn{3}{c}{\textbf{Weighted   Kendall's $\tau$}} &
  \multicolumn{3}{c}{\textbf{Kendall's $\tau$}} &
  \multicolumn{3}{c}{\textbf{Pearson}} \\  
\multirow{-2}{*}{\textbf{Target Dataset}} &
  \multicolumn{1}{c}{\textbf{MS}} &
  \multicolumn{1}{c}{\textbf{E}} &
  \textbf{Ours} &
  \multicolumn{1}{c}{\textbf{MS}} &
  \multicolumn{1}{c}{\textbf{E}} &
  \textbf{Ours} &
  \multicolumn{1}{c}{\textbf{MS}} &
  \multicolumn{1}{c}{\textbf{E}} &
  \textbf{Ours} \\ \midrule
Oxford102Flowers &
  \multicolumn{1}{c}{0.086} &
  \multicolumn{1}{c}{-0.019} &
  \textbf{0.616} &
  \multicolumn{1}{c}{0.138} &
  \multicolumn{1}{c}{0.074} &
  \textbf{0.400} &
  \multicolumn{1}{c}{0.230} &
  \multicolumn{1}{c}{0.164} &
  \textbf{0.456} \\ 
OxfordIIITPets &
  \multicolumn{1}{c}{0.414} &
  \multicolumn{1}{c}{0.393} &
  \textbf{0.558} &
  \multicolumn{1}{c}{0.346} &
  \multicolumn{1}{c}{0.326} &
  \textbf{0.453} &
  \multicolumn{1}{c}{0.504} &
  \multicolumn{1}{c}{0.500} &
  \textbf{0.666} \\ 
StanfordDogs &
  \multicolumn{1}{c}{0.326} &
  \multicolumn{1}{c}{0.323} &
  \textbf{0.477} &
  \multicolumn{1}{c}{0.244} &
  \multicolumn{1}{c}{0.242} &
  \textbf{0.427} &
  \multicolumn{1}{c}{0.398} &
  \multicolumn{1}{c}{0.407} &
  \textbf{0.604} \\ 
Caltech101 &
  \multicolumn{1}{c}{0.435} &
  \multicolumn{1}{c}{0.409} &
  \textbf{0.565} &
  \multicolumn{1}{c}{0.240} &
  \multicolumn{1}{c}{0.231} &
  \textbf{0.335} &
  \multicolumn{1}{c}{0.353} &
  \multicolumn{1}{c}{0.341} &
  \textbf{0.486} \\ 
StanfordCars &
  \multicolumn{1}{c}{0.115} &
  \multicolumn{1}{c}{0.018} &
  \textbf{0.486} &
  \multicolumn{1}{c}{0.137} &
  \multicolumn{1}{c}{0.071} &
  \textbf{0.368} &
  \multicolumn{1}{c}{0.256} &
  \multicolumn{1}{c}{0.163} &
  \textbf{0.549} \\ \midrule
Average &
  \multicolumn{1}{c}{0.275} &
  \multicolumn{1}{c}{0.225} &
  \textbf{0.540} &
  \multicolumn{1}{c}{0.221} &
  \multicolumn{1}{c}{0.190} &
  \textbf{0.367} &
  \multicolumn{1}{c}{0.348} &
  \multicolumn{1}{c}{0.315} &
  \textbf{0.552} \\ \bottomrule
\end{tabular}
}
\vspace{2mm}
\caption{Comparison of different ensemble transferability estimation metrics for fully-supervised models (classification tasks). The best results are indicated in bold. Note: MS: MS-LEEP, E: E-LEEP, Ours: OSBORN.}
\vspace{-3mm}
\label{table:fully_supervised}
\end{table*}

\vspace{3pt}
 \noindent \textbf{Evaluation on Fully-Supervised Pre-Trained Models.}
\label{sec:fullysupervised}
We herein compare our OSBORN with the baseline metrics, i.e. MS-LEEP and E-LEEP, in terms of three correlation metrics, WKT, KT, and PCC\footnote{Our baselines MS-LEEP and E-LEEP use custom proprietary model architectures that are not publicly available. We hence followed the authors' code and guidelines in using their method on the models used in our work, and picked the best-performing hyperparameters for the results corresponding to their baselines shown in this work.}. The correlation values are reported in Table \ref{table:fully_supervised}. Averaged across five target datasets, OSBORN improves $96.36\%$ over MS-LEEP and $140\%$ over E-LEEP in terms of WKT; improves $66.06\%$ over MS-LEEP and $93.16\%$ over E-LEEP in terms of KT; improves $58.62\%$ over MS-LEEP and $75.23\%$ over E-LEEP in terms of PCC. We can visually see the overall performance of our metric outperforming the existing baselines significantly in Fig \ref{fig:wkt_fully_supervised}.

\begin{figure}[t]
  \centering
  \scalebox{1}{
    \includegraphics[width=0.85\linewidth, height=6cm]
{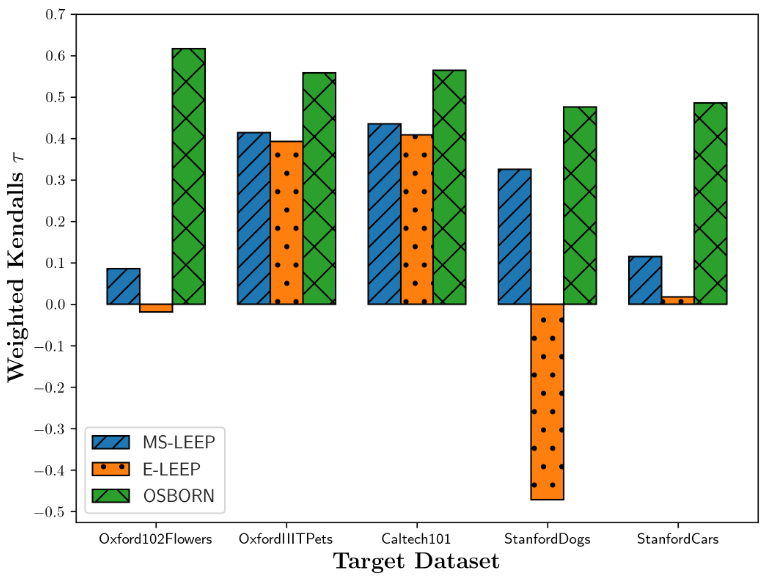}
  }
    \caption{Comparison of OSBORN over 5 target datasets interms of Weighted Kendalls's $\tau$. We can see that our metric constantly outperforms the baselines across every dataset by a large margin.}
    \vspace{-7pt}
  \label{fig:wkt_fully_supervised}
\end{figure}

\vspace{3pt}
 \noindent \textbf{Evaluation on Self-Supervised Pre-Trained Models.}\label{sec:self_supervised}
We compare the performance of our method with the baseline methods, i.e. MS-LEEP and E-LEEP. We present the experimental results regarding different correlation coefficients in Table \ref{table:self_supervised}. Note that we use the Frobenius norm regularizer while solving the OT problem because it gave us better results when compared to using other regularizers. In the appendix, we report results without any regularizers and compare them with the Frobenius norm variant. Table \ref{table:self_supervised} shows that, averaged across five target datasets, OSBORN improves $268.69\%$ over MS-LEEP and $231.82\%$ over E-LEEP in terms of WKT; improves $442.10\%$ over MS-LEEP and $379.07\%$ over E-LEEP in terms of KT; improves $527.27\%$ over MS-LEEP and $392.86\%$ over E-LEEP in terms of PCC.
%MS-LEEP and E-LEEP fail on some target datasets, such as Oxford-102Flowers and StanfordCars because they provide negative correlations. 

\begin{table*}[]
\centering
\scalebox{0.77}{
\begin{tabular}{c|ccc|ccc|ccc}
\toprule
 &
  \multicolumn{3}{c}{\textbf{Weighted Kendall's $\tau$}} &
  \multicolumn{3}{c}{\textbf{Kendall  's $\tau$}} &
  \multicolumn{3}{c}{\textbf{Pearson}} \\ 
\multirow{-2}{*}{\textbf{Target Dataset}} &
  \multicolumn{1}{c}{\textbf{MS}} &
  \multicolumn{1}{c}{\textbf{E}} &
  \textbf{Ours} &
  \multicolumn{1}{c}{\textbf{MS}} &
  \multicolumn{1}{c}{\textbf{E}} &
  \textbf{Ours} &
  \multicolumn{1}{c}{\textbf{MS}} &
  \multicolumn{1}{c}{\textbf{E}} &
  \textbf{Ours} \\ \midrule
Oxford102Flowers &
  \multicolumn{1}{c}{-0.080} &
  \multicolumn{1}{c}{-0.090} &
  \textbf{0.549} &
  \multicolumn{1}{c}{-0.035} &
  \multicolumn{1}{c}{-0.050} &
  \textbf{0.336} &
  \multicolumn{1}{c}{-0.077} &
  \multicolumn{1}{c}{-0.090} &
  \textbf{0.306} \\ 
OxfordIIITPets &
  \multicolumn{1}{c}{0.555} &
  \multicolumn{1}{c}{\textbf{0.574}} &
  0.357 &
  \multicolumn{1}{c}{0.221} &
  \multicolumn{1}{c}{\textbf{0.229}} &
  0.139 &
  \multicolumn{1}{c}{0.201} &
  \multicolumn{1}{c}{{0.212}} &
  \textbf{0.232} \\ 
StanfordDogs &
  \multicolumn{1}{c}{0.089} &
  \multicolumn{1}{c}{{0.132}} &
  \textbf{0.170} &
  \multicolumn{1}{c}{0.014} &
  \multicolumn{1}{c}{0.029} &
  \textbf{0.110} &
  \multicolumn{1}{c}{0.132} &
  \multicolumn{1}{c}{0.159} &
  \textbf{0.236} \\ 
Caltech101 &
  \multicolumn{1}{c}{0.290} &
  \multicolumn{1}{c}{0.311} &
  \textbf{0.488} &
  \multicolumn{1}{c}{0.195} &
  \multicolumn{1}{c}{0.228} &
  \textbf{0.308} &
  \multicolumn{1}{c}{0.248} &
  \multicolumn{1}{c}{0.287} &
  \textbf{0.374} \\ 
StanfordCars &
  \multicolumn{1}{c}{-0.359} &
  \multicolumn{1}{c}{-0.377} &
  \textbf{0.260} &
  \multicolumn{1}{c}{-0.207} &
  \multicolumn{1}{c}{-0.221} &
  \textbf{0.139} &
  \multicolumn{1}{c}{-0.285} &
  \multicolumn{1}{c}{-0.289} &
  \textbf{0.232} \\  \midrule
Average &
  \multicolumn{1}{c}{0.099} &
  \multicolumn{1}{c}{0.110} &
  \textbf{0.365} &
  \multicolumn{1}{c}{0.038} &
  \multicolumn{1}{c}{{0.043}} &
  \textbf{0.206} &
  \multicolumn{1}{c}{0.044} &
  \multicolumn{1}{c}{{0.056}} &
  \textbf{0.276} \\ \bottomrule
\end{tabular}

}
\vspace{2mm}
\caption{Comparison of different ensemble transferability estimation metrics for self-supervised pre-trained models (classification tasks). The best results are indicated in bold. Note: MS: MS-LEEP, E: E-LEEP and Ours: OSBORN.}
\vspace{-5mm}
\label{table:self_supervised}
\end{table*}

% \vspace{-5pt}
% \section{More Results and Analysis}
% \vspace{-3pt}
% \begin{enumerate}
%     \item we have to talk about the strengths of each component of the OSBORN. And give a reference to supplementary to study the auxiliary tasks.
% \end{enumerate}

\vspace{3pt}
\noindent \textbf{Performance of Selected Ensembles.}
\label{sec:ensembleeval}
Table \ref{table:greedy_optimal_accuracy} reports the ensemble accuracy of the models selected through OSBORN. For completeness of this discussion, we also report the same results for OSBORN without greedy maximization as well as for MS-LEEP and E-LEEP in Table \ref{table:ensemble_accuracies}. Following \cite{eleep}, we first calculate the OSBORN value for every ensemble candidate and pick the ensemble that bags the highest value. We follow a similar strategy with MS-LEEP and E-LEEP to pick the best model according to their values. To compute the ensemble accuracy, we used the individual models fine-tuned on the target train set and got their predictions on the target test set. We average these predictions and compare them with the ground truth labels to obtain overall accuracy.  We observe that the ensemble selected by OSBORN achieves the highest test accuracy across all datasets. In the case of both fully supervised and self-supervised settings, the baseline methods, i.e. MS-LEEP and E-LEEP, select the same ensembles (despite having different correlation values) in every case, which is why they obtain the same ensemble accuracy. %This process is computationally expensive as the number of ensemble candidates increases exponentially with the size of the source pool.  
 % We report the ensemble accuracies of Oxford 102 Flowers, Caltech-101, and Stanford Cars datasets. 

\vspace{3pt}
\noindent \textbf{Scaling Number of Models in Ensemble.}
As shown earlier in this section (Fig \ref{fig:ensemblesize}), we found the performance to saturate after an ensemble size of 3 in the datasets considered in this work as well as in \cite{eleep}. On the other hand, we also observe unsurprisingly that the cost of ensemble selection can go up significantly as the ensemble size increases. %This is because the time taken to calculate each model's gain increases with the number of models already present in the set. 
We show the cost performance of models selected for the Caltech101 dataset in Fig \ref{fig:tr1}. Despite the increasing trend, we note that the time taken is still in the order of seconds, which makes the proposed OSBORN metric practical and relevant. %Although the cost goes up as A greedy strategy This was another reason for our choice of $k=3$ in this work.

%stated in Sec \ref{sec:prop_metric}, our proposed metric can be viewed as a submodular set function (Eqn \ref{eq:scoring_func}) which allows us to use a simple greedy maximization algorithm to find a near-optimal ensemble of the required size. In this section, we report the performance of the ensemble of size 3 (on target test set) which we obtain using the greedy maximization algorithm. The results are in Table \ref{table:greedy_optimal_accuracy}. 

\begin{figure}[t]
  \centering
  \scalebox{0.75}{
\includegraphics[width=0.85\linewidth, height=5cm]
{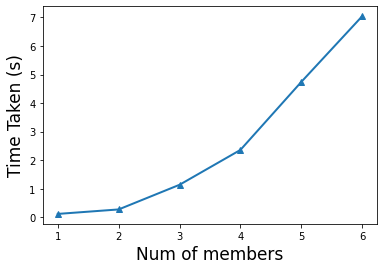}}
  \caption{Cost performace of model selection for the Caltech101 dataset.}
  \vspace{-15pt}
  \label{fig:tr1}
\end{figure}

\begin{table}[]
\scalebox{0.8}{
\begin{tabular}{cc|c|c}
\toprule
& \multicolumn{3}{|c}{\textbf{Ensemble Accuracy (Fully Supervised)}}    
\\ \cmidrule{2-4} 
\multirow{-2}{*}{\textbf{Target Dataset}}       & \multicolumn{1}{|c|}{\textbf{MS-LEEP}} & \multicolumn{1}{c|}{\textbf{E-LEEP}}                & \textbf{Ours}                       \\ \midrule
Oxford102Flowers        & \multicolumn{1}{|c|}{85.347}           & \multicolumn{1}{c|}{85.347} & \textbf{89.865}                           \\ 
Caltech101              & \multicolumn{1}{|c|}{\textbf{68.533}}  & \multicolumn{1}{c|}{\textbf{68.533}}                & \textbf{68.533}                                    \\ 
StanfordCars            & \multicolumn{1}{|c|}{48.623}           & \multicolumn{1}{c|}{48.623}                         & \textbf{62.915}                                    \\ \midrule
\textbf{Average}        & \multicolumn{1}{|c|}{67.501}      & \multicolumn{1}{c|}{67.501}                    & \textbf{73.771}                      \\ \midrule
& \multicolumn{3}{c}{\textbf{Ensemble Accuracy (Self Supervised)}}                                                                       \\ \midrule
Oxford102Flowers        & \multicolumn{1}{|c|}{88.278}           & \multicolumn{1}{c|}{88.278} & \textbf{93.040}                           \\ 
Caltech101              & \multicolumn{1}{|c|}{\textbf{86.933}}   & \multicolumn{1}{c|}{\textbf{86.933}}                 & \textbf{89.333}                           \\ 
StanfordCars            & \multicolumn{1}{|c|}{6.056}             & \multicolumn{1}{c|}{6.056}                           & \textbf{61.820}                           \\ \midrule
\textbf{Average}        & \multicolumn{1}{|c|}{60.422}      & \multicolumn{1}{c|}{60.422}                    & \multicolumn{1}{c}{\textbf{80.598}} \\ \bottomrule
\end{tabular}
}
\vspace{2mm}
\caption{We compare the target test set accuracies achieved by fine-tuned model ensembles picked by MS-LEEP, E-LEEP and OSBORN.}
\vspace{-3mm}
\label{table:ensemble_accuracies}
\end{table}

\vspace{3pt}
\noindent \textbf{Ablation Studies.} 
We conducted additional experiments to understand the influence of each component in OSBORN (included in the Appendix). In general, while simple addition of the three quantities in OSBORN without any weights outperformed previous methods, we observed that these can be finetuned through grid search over a larger range of values to get even better transferability estimates. This however varies with the target dataset.
On Caltech101 as the target dataset, we noticed that
%The influence of each component of OSBORN towards estimating the transferability varies slightly with the target dataset. for Caltech101 as the target dataset. We noticed that 
giving more weightage to $W_D$ compared to the other two terms ($W_T$ and $W_C$) achieved higher correlation scores, as shown in Fig \ref{fig:caltech_Wd}. This could be because of the wide variety of images in this dataset. $W_D$ measures the latent space mismatch between such varied images with the source datasets (which may not have overlapping images/representation with the target set), which benefits in this case. %reinforcing our finding of giving more weightage to the $W_D$ term for Caltech101. 
More detailed analysis is provided in the Appendix.

\begin{figure}[h!]
  \centering
  \scalebox{1}{
\includegraphics[width=1\columnwidth, height=2.4cm]
{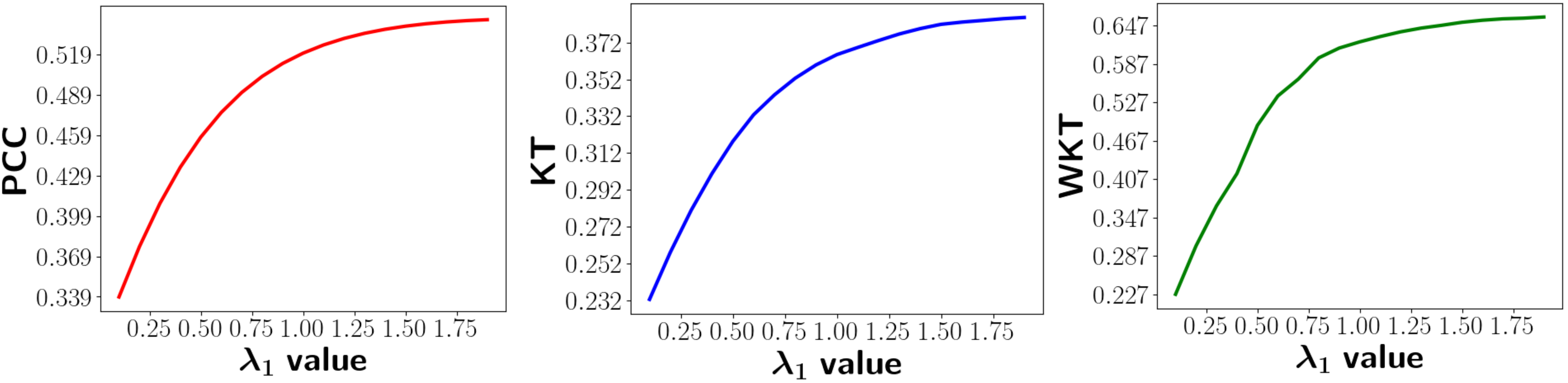}}
  \caption{$W_D$ weightage vs. correlation comparison for Caltech101. We set weights for $W_T$ and $W_C$ as 1.}
  \vspace{-10pt}
  \label{fig:caltech_Wd}
\end{figure}

\vspace{-2pt}
\tocless\section{Conclusions}
\vspace{-4pt}
In this paper, we propose a novel optimal transport-based transferability estimation metric, OSBORN, carefully designed for ensembles that consider multiple factors, such as the mismatch in the latent space, label space, and the cohesiveness amongst the individual models in the ensemble. We show that the proposed metric can be treated as a submodular optimization problem, allowing us to leverage a greedy strategy for source model ensemble selection. We show experimentally that our metric outperforms the existing metrics MS-LEEP and E-LEEP across tasks on multiple correlation metrics. Future directions include increasing the computational efficiency of this method, as well as studying its applicability to other tasks and problem settings.  
\vspace{2mm}
\section*{Acknowledgements}
\vspace{-4pt}
This work was partly supported by KLA and the Department of
Science and Technology, India through the DST ICPS Data Science Cluster program. We would like to thank the authors of \cite{eleep} for insightful discussions. Further, we thank the anonymous reviewers for their valuable feedback that improved the presentation of this paper.

%%%%%%%%% REFERENCES
{\small
\bibliographystyle{ieee_fullname}
\bibliography{main}
}

\newpage
\appendix

\section*{Appendix}
\noindent In this appendix, we provide additional details which we could not include in the main paper due to space constraints, including additional results, details and analysis that provide more insights into the proposed method. In particular, we discuss the following:

\tableofcontents
% \addtocontents{toc}{\protect\setcounter{tocdepth}{0}}
% \addtocontents{lof}{\protect\setcounter{tocdepth}{0}}
\section{Comparison against OTCE}
In this section, we compare OSBORN with the OTCE metric. OTCE is limited by its ability to estimate transferability for a single source model; however, we naively add the OTCE scores of the individual models present in the ensemble to make it a multi-source variant. The results in terms of various correlations are shown in Tab. \ref{table:otcevsosborn}. OSBORN outperforms OTCE by $131.76\%$ in terms of WKT, $235.59\%$ in terms of KT and $513.33\%$ in terms of PCC. 

\section{Modified Baselines}
In this section, we understand the effect of adding the model cohesion term $W_{C}$ to our baselines i.e. MS-LEEP and E-LEEP. Table \ref{table:mod_baselines} shows the results. While it expectedly improves correlations of these baselines (further corroborating the usefulness of our proposed cohesiveness term), OSBORN still achieves higher correlations than these modified baselines.

\begin{table}[h!]
\centering
\scalebox{0.7}{
\begin{tabular}{c|cc|cc|cc}
\toprule
\multirow{2}{*}{\textbf{Target Dataset}} &
  \multicolumn{2}{c}{\textbf{Weighted Kendall's $\tau$}} &
  \multicolumn{2}{c}{\textbf{Kendall's $\tau$}} &
  \multicolumn{2}{c}{\textbf{Pearson}} \\ 
 &
  \multicolumn{1}{c}{\textbf{OTCE}} &
  \textbf{Ours} &
  \multicolumn{1}{c}{\textbf{OTCE}} &
  \textbf{Ours} &
  \multicolumn{1}{c}{\textbf{OTCE}} &
  \textbf{Ours} \\ \midrule
Oxford102Flowers &
  \multicolumn{1}{c}{0.406} &
  \textbf{0.616} &
  \multicolumn{1}{c}{0.118} &
  \textbf{0.400} &
  \multicolumn{1}{c}{0.086} &
  \textbf{0.456} \\ 
OxfordIIITPets &
  \multicolumn{1}{c}{0.186} &
  \textbf{0.558} &
  \multicolumn{1}{c}{0.075} &
  \textbf{0.453} &
  \multicolumn{1}{c}{0.109} &
  \textbf{0.666} \\ 
StanfordDogs &
  \multicolumn{1}{c}{0.093} &
  \textbf{0.477} &
  \multicolumn{1}{c}{0.05} &
  \textbf{0.427} &
  \multicolumn{1}{c}{0.088} &
  \textbf{0.604} \\ 
Caltech101 &
  \multicolumn{1}{c}{0.179} &
  \textbf{0.565} &
  \multicolumn{1}{c}{0.223} &
  \textbf{0.335} &
  \multicolumn{1}{c}{0.068} &
  \textbf{0.486} \\ 
StanfordCars &
  \multicolumn{1}{c}{0.300} &
  \textbf{0.486} &
  \multicolumn{1}{c}{0.123} &
  \textbf{0.368} &
  \multicolumn{1}{c}{0.100} &
  \textbf{0.549} \\ \midrule
Average &
  \multicolumn{1}{c}{0.233} &
  \textbf{0.540} &
  \multicolumn{1}{c}{0.118} &
  \textbf{0.396} &
  \multicolumn{1}{c}{0.090} &
  \textbf{0.552} \\ \bottomrule
\end{tabular}}
\vspace{2mm}
\caption{OTCE vs OSBORN (Ours)}
\label{table:otcevsosborn}
\end{table}

\begin{figure*}
  \centering
\includegraphics[width=\textwidth]
{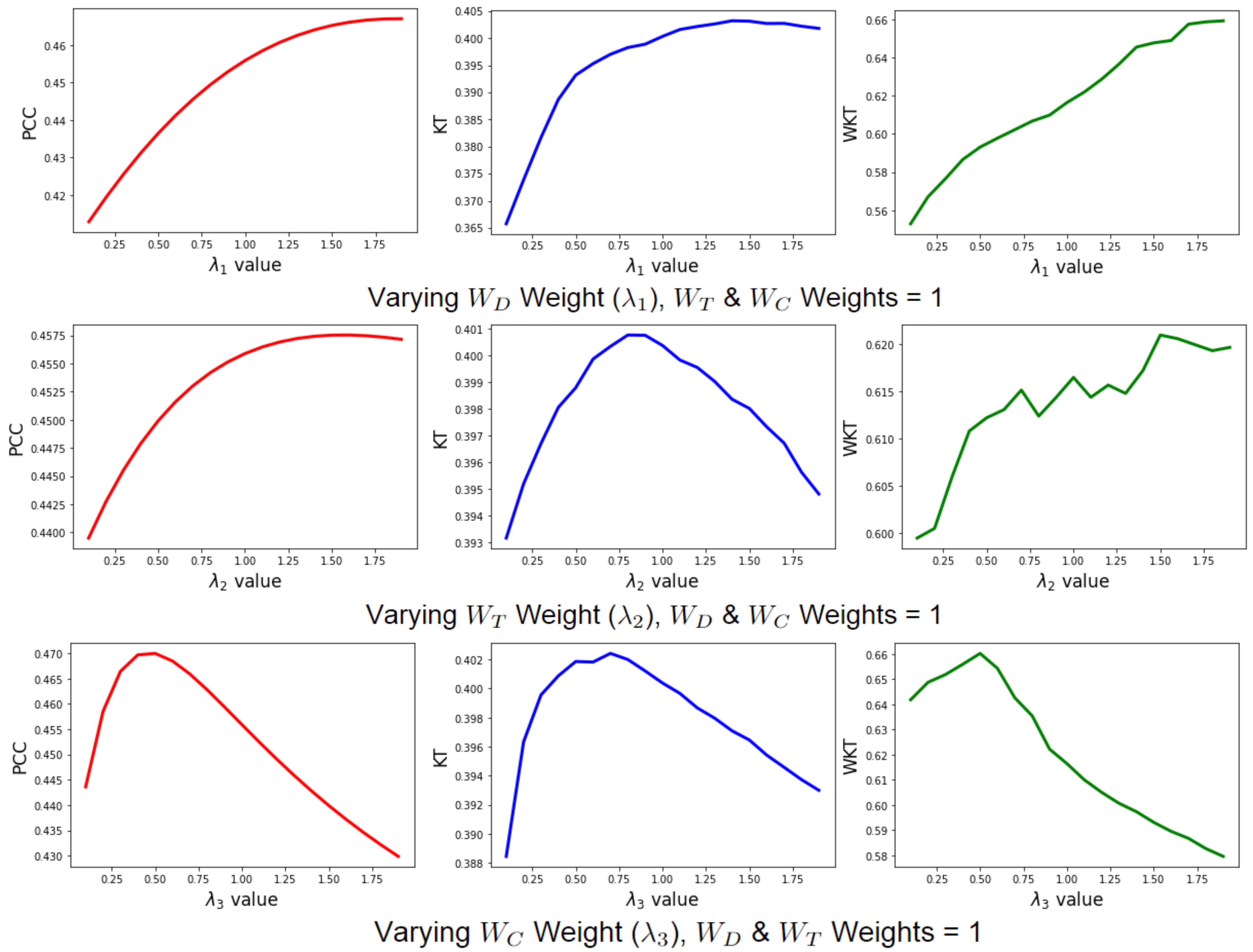}
  \caption{Relation between weighted coefficient values for terms in OSBORN and corresponding correlation scores for Oxford102Flower}
  \label{fig:oxford102flower}
\end{figure*}

% \begin{figure*}
%   \centering
% \includegraphics[width=\textwidth]
% {iccv2023AuthorKit/images/tSNE_Plots_FSL.png}
%   \caption{t-SNE plots of features learned by corresponding method's ensembles on StanfordCars dataset. `\textit{Optimal}' chooses best ensemble with exhaustive search}
%   \label{fig:tsne}
% \end{figure*}

\begin{table*}[]
\scalebox{0.77}{
\begin{tabular}{c|ccccc|ccccc|ccccc}
\toprule
\multirow{2}{*}{\textbf{Target Dataset}} &
  \multicolumn{5}{c}{\textbf{Weighted   Kendall's $\tau$}} &
  \multicolumn{5}{c}{\textbf{Kendall's $\tau$}} &
  \multicolumn{5}{c}{\textbf{Pearson}} \\ \cmidrule{2-16} 
 &
  \multicolumn{1}{c}{\textbf{MS}} &
  \multicolumn{1}{c}{\textbf{E}} &
  \multicolumn{1}{c}{\textbf{IoU}} &
  \multicolumn{1}{c}{\textbf{sIoU}} &
  \textbf{Ours} &
  \multicolumn{1}{c}{\textbf{MS}} &
  \multicolumn{1}{c}{\textbf{E}} &
  \multicolumn{1}{c}{\textbf{IoU}} &
  \multicolumn{1}{c}{\textbf{sIoU}} &
  \textbf{Ours} &
  \multicolumn{1}{c}{\textbf{MS}} &
  \multicolumn{1}{c}{\textbf{E}} &
  \multicolumn{1}{c}{\textbf{IoU}} &
  \multicolumn{1}{c}{\textbf{sIoU}} &
  \textbf{Ours} \\ \midrule
Camvid    &
  \multicolumn{1}{c}{-0.173} &
  \multicolumn{1}{c}{-0.279} &
  \multicolumn{1}{c}{0.175} &
  \multicolumn{1}{c}{-0.074} &
  \textbf{0.190} &
  \multicolumn{1}{c}{-0.006} &
  \multicolumn{1}{c}{-0.108} &
  \multicolumn{1}{c}{0.030} &
  \multicolumn{1}{c}{-0.050} &
  \textbf{0.114} &
  \multicolumn{1}{c}{0.088} &
  \multicolumn{1}{c}{-0.050} &
  \multicolumn{1}{c}{0.071} &
  \multicolumn{1}{c}{-0.024} &
  \textbf{0.091} \\ 
Cityscapes    &
  \multicolumn{1}{c}{-0.356} &
  \multicolumn{1}{c}{-0.390} &
  \multicolumn{1}{c}{-0.306} &
  \multicolumn{1}{c}{-0.153} &
  \textbf{0.056} &
  \multicolumn{1}{c}{-0.166} &
  \multicolumn{1}{c}{-0.188} &
  \multicolumn{1}{c}{-0.115} &
  \multicolumn{1}{c}{-0.090} &
  \textbf{0.108} &
  \multicolumn{1}{c}{-0.263} &
  \multicolumn{1}{c}{-0.241} &
  \multicolumn{1}{c}{-0.191} &
  \multicolumn{1}{c}{-0.154} &
  \textbf{0.216} \\ 
SUIM    &
  \multicolumn{1}{c}{0.052} &
  \multicolumn{1}{c}{0.051} &
  \multicolumn{1}{c}{0.191} &
  \multicolumn{1}{c}{0.097} &
  \textbf{0.237} &
  \multicolumn{1}{c}{-0.014} &
  \multicolumn{1}{c}{-0.016} &
  \multicolumn{1}{c}{\textbf{0.084}} &
  \multicolumn{1}{c}{0.075} &
  0.078 &
  \multicolumn{1}{c}{-0.024} &
  \multicolumn{1}{c}{-0.028} &
  \multicolumn{1}{c}{\textbf{0.230}} &
  \multicolumn{1}{c}{0.164} &
  0.112 \\ \midrule
Average &
  \multicolumn{1}{c}{-0.159} &
  \multicolumn{1}{c}{-0.053} &
  \multicolumn{1}{c}{0.020} &
  \multicolumn{1}{c}{-0.043} &
  \textbf{0.161} &
  \multicolumn{1}{c}{-0.062} &
  \multicolumn{1}{c}{-0.104} &
  \multicolumn{1}{c}{0.0003} &
  \multicolumn{1}{c}{-0.022} &
  \textbf{0.1} &
  \multicolumn{1}{c}{-0.066} &
  \multicolumn{1}{c}{-0.106} &
  \multicolumn{1}{c}{0.037} &
  \multicolumn{1}{c}{-0.005} &
  \textbf{0.140} \\ \bottomrule
\end{tabular}
}
 \vspace{1mm}
\caption{Comparison of different ensemble transferability estimation metrics for 
 semantic segmentation tasks. On average, we beat all the previously proposed methods for estimating transferability for semantic segmentation in terms of correlations. Note, MS: MS-LEEP, E: E-LEEP, IoU: IoU-EEP, sIoU: SoftIoU-EEP.}
 % \vspace{-4.5mm}
\label{tab:sem_seg}
\end{table*}

\begin{table*}[h]
\centering
\scalebox{0.75}{
\begin{tabular}{c|cccc|cccc|cccc}
\toprule
 &
  \multicolumn{4}{c}{\textbf{Weighted   Kendall's $\tau$}} &
  \multicolumn{4}{c}{\textbf{Kendall's   $\tau$}} &
  \multicolumn{4}{c}{\textbf{Pearson}} \\  
\multirow{-2}{*}{\textbf{Target Dataset}} &
  \multicolumn{1}{c}{\textbf{MS}} &
  \multicolumn{1}{c}{\textbf{E}} &
  \multicolumn{1}{c}{\textbf{$W_C$   + MS}} &
  \textbf{$W_C$   + E} &
  \multicolumn{1}{c}{\textbf{MS}} &
  \multicolumn{1}{c}{\textbf{E}} &
  \multicolumn{1}{c}{\textbf{$W_C$   + MS}} &
  \textbf{$W_C$   + E} &
  \multicolumn{1}{c}{\textbf{MS}} &
  \multicolumn{1}{c}{\textbf{E}} &
  \multicolumn{1}{c}{\textbf{$W_C$   + MS}} &
  \textbf{$W_C$   + E} \\ \midrule
Oxford102Flowers &
  \multicolumn{1}{c}{0.086} &
  \multicolumn{1}{c}{-0.019} &
  \multicolumn{1}{c}{0.413} &
  \textbf{0.459} &
  \multicolumn{1}{c}{0.138} &
  \multicolumn{1}{c}{0.0739} &
  \multicolumn{1}{c}{0.315} &
  \textbf{0.330} &
  \multicolumn{1}{c}{0.23} &
  \multicolumn{1}{c}{0.164} &
  \multicolumn{1}{c}{\textbf{0.401}} &
  0.385 \\ 
OxfordIIITPets &
  \multicolumn{1}{c}{0.414} &
  \multicolumn{1}{c}{0.393} &
  \multicolumn{1}{c}{\textbf{0.540}} &
  0.522 &
  \multicolumn{1}{c}{0.346} &
  \multicolumn{1}{c}{0.326} &
  \multicolumn{1}{c}{0.473} &
  \textbf{0.475} &
  \multicolumn{1}{c}{0.504} &
  \multicolumn{1}{c}{0.5} &
  \multicolumn{1}{c}{0.666} &
  \textbf{0.676} \\ 
Caltech101 &
  \multicolumn{1}{c}{\textbf{0.435}} &
  \multicolumn{1}{c}{0.409} &
  \multicolumn{1}{c}{0.314} &
  0.385 &
  \multicolumn{1}{c}{0.240} &
  \multicolumn{1}{c}{0.231} &
  \multicolumn{1}{c}{\textbf{0.242}} &
  0.236 &
  \multicolumn{1}{c}{0.353} &
  \multicolumn{1}{c}{0.341} &
  \multicolumn{1}{c}{0.315} &
  \textbf{0.354} \\ 
StanfordDogs &
  \multicolumn{1}{c}{0.326} &
  \multicolumn{1}{c}{-0.472} &
  \multicolumn{1}{c}{0.348} &
  \textbf{0.384} &
  \multicolumn{1}{c}{0.244} &
  \multicolumn{1}{c}{-0.236} &
  \multicolumn{1}{c}{0.269} &
  \textbf{0.326} &
  \multicolumn{1}{c}{0.398} &
  \multicolumn{1}{c}{-0.154} &
  \multicolumn{1}{c}{0.496} &
  \textbf{0.571} \\ 
StanfordCars &
  \multicolumn{1}{c}{0.115} &
  \multicolumn{1}{c}{0.018} &
  \multicolumn{1}{c}{0.066} &
  \textbf{0.147} &
  \multicolumn{1}{c}{0.137} &
  \multicolumn{1}{c}{0.071} &
  \multicolumn{1}{c}{0.144} &
  \textbf{0.185} &
  \multicolumn{1}{c}{0.256} &
  \multicolumn{1}{c}{0.163} &
  \multicolumn{1}{c}{0.360} &
  \textbf{0.434} \\ \midrule
Average &
  \multicolumn{1}{c}{0.275} &
  \multicolumn{1}{c}{0.097} &
  \multicolumn{1}{c}{0.265} &
  \textbf{0.301} &
  \multicolumn{1}{c}{0.221} &
  \multicolumn{1}{c}{0.110} &
  \multicolumn{1}{c}{0.246} &
  \textbf{0.259} &
  \multicolumn{1}{c}{0.348} &
  \multicolumn{1}{c}{0.222} &
  \multicolumn{1}{c}{0.383} &
  \textbf{0.407} \\ \bottomrule
\end{tabular}
}
\vspace{2mm}
\caption{Comparison of baselines and modified baselines. Note: MS: MS-LEEP, E: E-LEEP, $W_{C}$: Model Cohesion term}
\label{table:mod_baselines}
\end{table*}

\begin{table*}[h]
\centering
\scalebox{0.75}{
\begin{tabular}{ccccccc}
\toprule
\multicolumn{1}{c|}{\multirow{2}{*}{\textbf{Target Dataset}}} &
  \multicolumn{2}{c}{\textbf{Weighted Kendall's $\tau$}} &
  \multicolumn{2}{c}{\textbf{Kendall's $\tau$}} &
  \multicolumn{2}{c}{\textbf{Pearson}} \\ 
\multicolumn{1}{c|}{} &
  \multicolumn{1}{c}{\textbf{Standard}} &
  \multicolumn{1}{c|}{\textbf{Frobenius}} &
  \multicolumn{1}{c}{\textbf{Standard}} &
  \multicolumn{1}{c|}{\textbf{Frobenius}} &
  \multicolumn{1}{c}{\textbf{Standard}} &
  \textbf{Frobenius} \\ \midrule
\multicolumn{1}{c|}{Oxford102Flowers} &
  \multicolumn{1}{c}{\textbf{0.616}} &
  \multicolumn{1}{c|}{0.614} &
  \multicolumn{1}{c}{\textbf{0.400}} &
  \multicolumn{1}{c|}{0.390} &
  \multicolumn{1}{c}{\textbf{0.456}} &
  0.463 \\ 
\multicolumn{1}{c|}{OxfordIIITPets} &
  \multicolumn{1}{c}{\textbf{0.558}} &
  \multicolumn{1}{c|}{0.539} &
  \multicolumn{1}{c}{\textbf{0.453}} &
  \multicolumn{1}{c|}{0.446} &
  \multicolumn{1}{c}{\textbf{0.666}} &
  0.660 \\ 
\multicolumn{1}{c|}{Caltech101} &
  \multicolumn{1}{c}{\textbf{0.565}} &
  \multicolumn{1}{c|}{0.557} &
  \multicolumn{1}{c}{\textbf{0.335}} &
  \multicolumn{1}{c|}{0.329} &
  \multicolumn{1}{c}{\textbf{0.486}} &
  0.483 \\ 
\multicolumn{1}{c|}{StanfordDogs} &
  \multicolumn{1}{c}{0.477} &
  \multicolumn{1}{c|}{\textbf{0.581}} &
  \multicolumn{1}{c}{0.427} &
  \multicolumn{1}{c|}{\textbf{0.508}} &
  \multicolumn{1}{c}{0.604} &
  \textbf{0.628} \\ 
\multicolumn{1}{c|}{StanfordCars} &
  \multicolumn{1}{c}{\textbf{0.486}} &
  \multicolumn{1}{c|}{0.445} &
  \multicolumn{1}{c}{\textbf{0.368}} &
  \multicolumn{1}{c|}{0.361} &
  \multicolumn{1}{c}{\textbf{0.549}} &
  0.544 \\ \midrule
\multicolumn{1}{c|}{Average} &
  \multicolumn{1}{c}{0.540} &
  \multicolumn{1}{c|}{\textbf{0.547}} &
  \multicolumn{1}{c}{0.397} &
  \multicolumn{1}{c|}{\textbf{0.407}} &
  \multicolumn{1}{c}{0.552} &
  \textbf{0.556} \\ \bottomrule
\end{tabular}}
\vspace{2mm}
\caption{In this table, we report the change in correlations obtained using a Frobenius norm based regularizer rather than a standard (non-regularized) method for the fully-supervised pre-trained models (classification tasks).}
\label{table:regularizer}
\end{table*}
\begin{figure*}[h!]
  \centering
\includegraphics[width=\textwidth]
{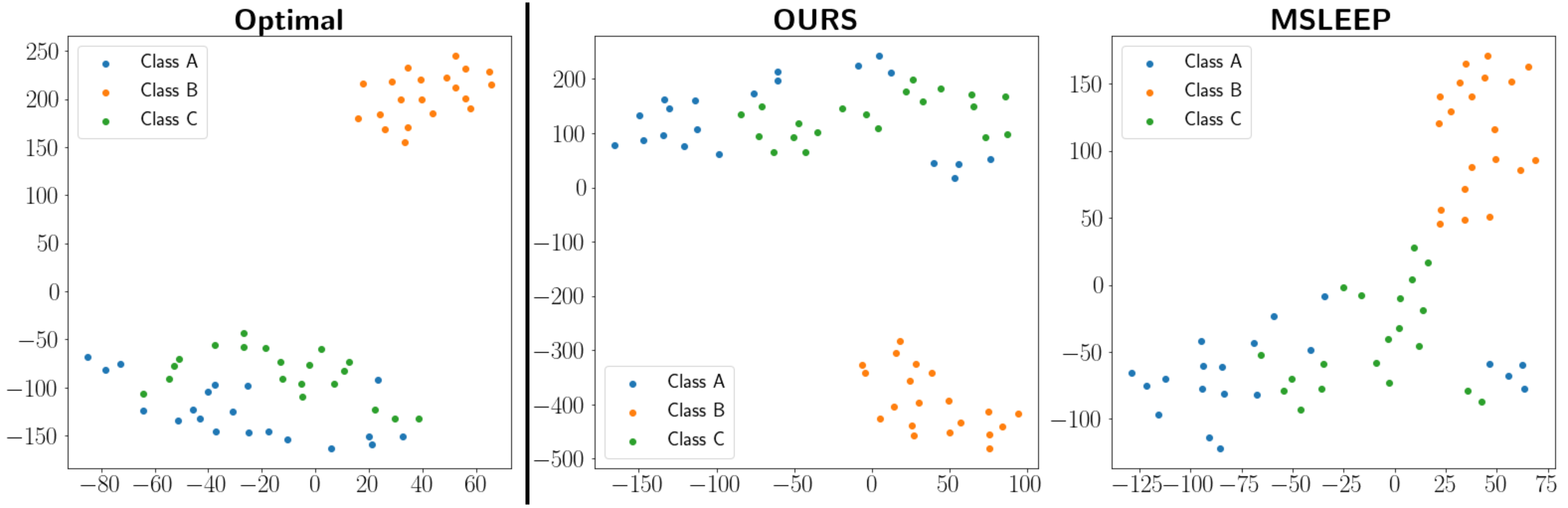}
  \caption{t-SNE plots of features learned by corresponding method's ensembles on StanfordCars dataset. `\textit{Optimal}' chooses best ensemble with exhaustive search}
  \label{fig:tsne}
\end{figure*}

\begin{table*}[]
\footnotesize
\centering
\begin{tabular}{c|ccc|ccc|ccc}
\toprule
 &
  \multicolumn{3}{c}{\textbf{Weighted   Kendall's $\tau$}} &
  \multicolumn{3}{c}{\textbf{Kendall's $\tau$}} &
  \multicolumn{3}{c}{\textbf{Pearson}} \\  
\multirow{-2}{*}{\textbf{Target Domain}} &
  \multicolumn{1}{c}{\textbf{MS}} &
  \multicolumn{1}{c}{\textbf{E}} &
  \textbf{Ours} &
  \multicolumn{1}{c}{\textbf{MS}} &
  \multicolumn{1}{c}{\textbf{E}} &
  \textbf{Ours} &
  \multicolumn{1}{c}{\textbf{MS}} &
  \multicolumn{1}{c}{\textbf{E}} &
  \textbf{Ours} \\ \midrule
Real &
  \multicolumn{1}{c}{0.057} &
  \multicolumn{1}{c}{0.026} &
  \textbf{0.576} &
  \multicolumn{1}{c}{0.016} &
  \multicolumn{1}{c}{-0.011} &
  \textbf{0.415} &
  \multicolumn{1}{c}{0.010} &
  \multicolumn{1}{c}{-0.033} &
  \textbf{0.518} \\ 
Infograph &
  \multicolumn{1}{c}{0.165} &
  \multicolumn{1}{c}{0.163} &
  \textbf{0.298} &
  \multicolumn{1}{c}{0.046} &
  \multicolumn{1}{c}{0.048} &
  \textbf{0.230} &
  \multicolumn{1}{c}{0.076} &
  \multicolumn{1}{c}{0.057} &
  \textbf{0.308} \\ 
Clipart &
  \multicolumn{1}{c}{0.003} &
  \multicolumn{1}{c}{-0.076} &
  \textbf{0.040} &
  \multicolumn{1}{c}{0.115} &
  \multicolumn{1}{c}{0.078} &
  \textbf{0.161} &
  \multicolumn{1}{c}{\textbf{0.248}} &
  \multicolumn{1}{c}{0.193} &
  0.179 \\ \midrule
Average &
  \multicolumn{1}{c}{0.075} &
  \multicolumn{1}{c}{0.038} &
  \textbf{0.305} &
  \multicolumn{1}{c}{0.059} &
  \multicolumn{1}{c}{0.038} &
  \textbf{0.269} &
  \multicolumn{1}{c}{0.111} &
  \multicolumn{1}{c}{0.072} &
  \textbf{0.335} \\ \bottomrule
\end{tabular}
\vspace{2mm}
\caption{Comparison of different ensemble transferability estimation metrics for classification tasks on the DomainNet dataset.
Averaged across 3 domains, OSBORN achieves the best results under all the correlation values. MS: MS-LEEP, and E: E-LEEP.}
\label{tab:domainadaptation}
\end{table*}
\begin{table*}[h!]
\centering
\scalebox{0.75}{
\begin{tabular}{ccccccc}
\toprule
\multicolumn{1}{c|}{\multirow{2}{*}{\textbf{Target Dataset}}} &
  \multicolumn{2}{c}{\textbf{Weighted Kendall's $\tau$}} &
  \multicolumn{2}{c}{\textbf{Kendall's $\tau$}} &
  \multicolumn{2}{c}{\textbf{Pearson}} \\ 
\multicolumn{1}{c|}{} &
  \multicolumn{1}{c}{\textbf{Standard}} &
  \multicolumn{1}{c|}{\textbf{Frobenius}} &
  \multicolumn{1}{c}{\textbf{Standard}} &
  \multicolumn{1}{c|}{\textbf{Frobenius}} &
  \multicolumn{1}{c}{\textbf{Standard}} &
  \textbf{Frobenius} \\ \midrule
\multicolumn{1}{c|}{Oxford102Flowers} &
  \multicolumn{1}{c}{0.492} &
  \multicolumn{1}{c|}{\textbf{0.549}} &
  \multicolumn{1}{c}{0.293} &
  \multicolumn{1}{c|}{\textbf{0.336}} &
  \multicolumn{1}{c}{0.272} &
  \textbf{0.306} \\ 
\multicolumn{1}{c|}{OxfordIIITPets} &
  \multicolumn{1}{c}{0.316} &
  \multicolumn{1}{c|}{\textbf{0.357}} &
  \multicolumn{1}{c}{0.123} &
  \multicolumn{1}{c|}{\textbf{0.139}} &
  \multicolumn{1}{c}{0.193} &
  \textbf{0.232} \\ 
\multicolumn{1}{c|}{StanfordDogs} &
  \multicolumn{1}{c}{0.140} &
  \multicolumn{1}{c|}{\textbf{0.170}} &
  \multicolumn{1}{c}{0.074} &
  \multicolumn{1}{c|}{\textbf{0.110}} &
  \multicolumn{1}{c}{0.210} &
  \textbf{0.236} \\ 
\multicolumn{1}{c|}{Caltech101} &
  \multicolumn{1}{c}{0.484} &
  \multicolumn{1}{c|}{\textbf{0.488}} &
  \multicolumn{1}{c}{0.279} &
  \multicolumn{1}{c|}{\textbf{0.308}} &
  \multicolumn{1}{c}{0.345} &
  \textbf{0.374} \\ 
\multicolumn{1}{c|}{StanfordCars} &
  \multicolumn{1}{c}{0.207} &
  \multicolumn{1}{c|}{\textbf{0.260}} &
  \multicolumn{1}{c}{0.100} &
  \multicolumn{1}{c|}{\textbf{0.139}} &
  \multicolumn{1}{c}{0.198} &
  \textbf{0.232} \\ \midrule
\multicolumn{1}{c|}{Average} &
  \multicolumn{1}{c}{0.328} &
  \multicolumn{1}{c|}{\textbf{0.365}} &
  \multicolumn{1}{c}{0.174} &
  \multicolumn{1}{c|}{\textbf{0.206}} &
  \multicolumn{1}{c}{0.244} &
  \textbf{0.276} \\ \bottomrule
\end{tabular}}
\vspace{2mm}
\caption{In this table, we understand the difference in correlations obtained using a Frobenius norm-based regularizer rather than a standard (non-regularized) method for the self-supervised pre-trained models (classification tasks).}
\label{table:reg_selfsupervised}
\end{table*}

\section{Additional Experiments}
In this section, we present the results of additional experiments we conducted on tasks like multi-domain/domain adaptation and semantic segmentation. We could not include details about these in the main paper due to space constraints. We start by describing the datasets used, models trained and then report the performance of OSBORN and other baselines on these tasks.

\noindent \textbf{Multi-domain/Domain Adaptation Dataset: DomainNet.}
We use the DomainNet \cite{domainnet} dataset to test OSBORN in a challenging multi-domain source pool setting. DomainNet consists of 6 domains (styles) namely, Clipart (C), Infograph (I), Painting (P), Quickdraw (Q), Real (R) and Sketch (S), each covering 345 common object categories. Out of these 6 domains, we evaluate the performance of OSBORN on 3 domains, that are Real (R), Infograph (I) and Clipart (C).

\noindent \textbf{Semantic Segmentation Datasets.} For conducting experiments on the semantic segmentation tasks, we choose 10 popularly used segmentation datasets, Pascal Context \cite{pascalcontext}, Pascal VOC \cite{pascalvoc}, COCO \cite{coco}, CamVid \cite{camvid}, CityScapes \cite{cityscapes}, India Driving Dataset (IDD) \cite{idd}, Berkeley Deep Drive (BDD) \cite{bdd}, Mapillary Vistas \cite{mapillary}, SUIM \cite{suim}, and SUN RGB-D \cite{sunrgbd}. Out of these 10 datasets, we evaluate and compare the performance of OSBORN with baselines on 3 target datasets, namely Camvid \cite{camvid}, CityScapes \cite{cityscapes}, and SUIM \cite{suim}.

\noindent \textbf{Model Architectures (DomainNet).} For building the source pool for the multi-domain experiments, we use the same models as we used in the fully-supervised pre-training setting i.e ResNet-101 \cite{resnet} and DenseNet-201 \cite{densenet}. Initially, both models are initialized with the fully-supervised ImageNet weights \cite{imagenet}, and we then train them on 6 domains of the DomainNet dataset.
\begin{table}[t]
\scalebox{0.68}{
\begin{tabular}{c|c|c|c|c}
\toprule
\textbf{Target Dataset} & \textbf{$\mathbf{W_D   + W_T + W_C}$}          & \textbf{$\mathbf{W_D   + W_T}$} & \textbf{$\mathbf{W_D   + W_C}$}                 & \textbf{$\mathbf{W_T   + W_C}$} \\ \midrule
OxfordIIITPets   & \textbf{0.666} & 0.539 & \underline{0.657}          & 0.622 \\ 
Oxford102Flowers & \textbf{0.455} & 0.418 & \underline{0.435}          & 0.405 \\ 
StanfordCars     & \textbf{0.548} & 0.524 & \underline{0.526}          & 0.512 \\ 
StanfordDogs     & \underline{0.604}          & 0.496 & \textbf{0.643} & 0.563 \\
Caltech101       & 0.486          & \underline{0.501} & \textbf{0.517} & 0.309 \\ \midrule
Average                 & {\underline{0.552}} & 0.496                 & {\textbf{0.556}} & 0.482\\
\bottomrule
\end{tabular}
}
\vspace{-1pt}
\caption{Comparison of pearson corr. scores. \textbf{Bold} represents highest score, \underline{Underline} represents second highest score.}
\vspace{-14pt}
\label{table:comparison_table1}
\end{table}
\noindent \textbf{Model Architectures (Semantic Segmentation).} For semantic segmentation, we employ a FCN \cite{semseglong} with ResNet-101 \cite{resnet} backbone, and a Lite R-ASPP with MobileNetv3 backbone \cite{lrasppmob} as our source model architectures. The capacity of the former is much higher than the latter thus bringing in diversity. We initialize these models with the COCO pre-trained weights \cite{coco} and then train them on the 10 datasets to include them in our source pool \footnote{Our baselines MS-LEEP and E-LEEP use custom proprietary model architectures that are not publicly available. We hence followed the authors' code and obtained guidelines from them in using their method on the models used in our work, and picked the best-performing hyperparameters for the results corresponding to their baselines shown in this work.}.%\textbf{Note:} 
%\footnote{Our baselines MS-LEEP and E-LEEP use custom proprietary model architectures that are not publicly available. We hence followed the authors' code and guidelines in using their method on the models used in our work, and picked the best-performing hyperparameters for the results corresponding to their baselines shown in this work.}.
The rest of the experimental setup is the same as in Section 5 of the main paper.

\noindent \textbf{Results on DomainNet.} We compare OSBORN with the baseline metrics, i.e. MS-LEEP and E-LEEP, in terms of WKT, KT, and PCC. The correlation values are
reported in Tab. \ref{tab:domainadaptation}, averaged across three target domains.

\noindent \textbf{Results on Semantic Segmentation.}
\label{sec:sem_seg}
Apart from MS-LEEP and E-LEEP, the paper \cite{eleep} also proposes two additional metrics for predicting transferability on semantic segmentation tasks, which are namely IoU-EEP and SoftIoU-EEP. In this section, we compare the performance of OSBORN with these two metrics as well. We present the experimental results for the semantic segmentation tasks in Tab. \ref{tab:sem_seg}. As seen in the table, OSBORN improves transferability estimation when compared to previous works.

\section{Weighted version of OSBORN}
\vspace{-3pt}

While our results in the main paper showed that OSBORN outperforms existing state-of-the-art as is in its simple form, we conducted additional experiments to study the influence of weighting each component of OSBORN. Our studies showed that this can vary for different target datasets. Fig. \ref{fig:oxford102flower} shows these results for the Oxford102Flowers dataset. For target datasets such as OxfordIIITPets and Oxford102Flowers,  we observe that when we give more weightage to $W_D$ and subsequently to $W_T$, as compared to $W_C$, we achieve higher correlations. We believe this is because these datasets have some fine-grained characteristics in each class, which need more attention for classification. We believe that such a trend holds for transfer from coarse-grained to fine-grained datasets in general, while we observed a higher weightage to $W_T$ to provide more favorable results in other settings. % This explains why a higher weightage for $W_T$ compared to $W_C$ makes sense. 
% There are slight variations in correlation with different values of $\lambda$s for each term in OSBORN across datasets. 
As stated earlier, while not using any weighted coefficients for the terms in OSBORN is by itself beneficial, carefully picking weights for a specific target dataset can further improve performance. Learning these weighting coefficients would be an interesting direction for future work.

% It must be noted that a simple addition of the three quantities in OSBORN without any weights outperformed previous methods in all the datasets.
% But for the target datasets such as StanfordCars, we notice giving more weightage to $W_T$ than to $W_D$ achieves higher correlation. This is due to the images in the StanfordCars majoritively having the same set of features, which makes it obvious to give lower weightage to $W_D$, as it helps very little to differentiate different classes. Giving higher weightage to $W_T$ is more critical to differentiate the classes.

\vspace{-3mm}
\section{Visualization of Results}
In Fig \ref{fig:tsne}, we show t-SNE plots for data points of different classes in StanfordCars when passed through ensembles selected using various methods. We see that the ensemble selected by our method is better at segregating classes and closer to the Optimal as compared to MS-LEEP.

\section{Results with Frobenius Norm Regularizer}
\vspace{-3pt}
As mentioned in Section 3 of the paper, there is an option to use a regularizer to solve the OT problem. In this section, we investigate the usage of a Frobenius norm regularizer \cite{cgot},\cite{regot} in the experiments for image classification tasks (both fully-supervised and self-supervised pre-training settings). In Tab. \ref{table:regularizer}, we show the results of OSBORN with the use of a Frobenius norm regularizer (column: Frobenius) and without any regularizer (column: Standard) for the fully-supervised pre-training setting. We observe that both variations give comparable results on an average. In Tab. \ref{table:reg_selfsupervised}, we report the results for a self-supervised pre-training setting. In contrast to Tab. \ref{table:regularizer}, we observe that a Frobenius norm regularizer improves the performance substantially in this case. We hypothesize that self-supervised pre-training may make a model more conducive to the source datasets, which a Frobenius norm regularizer offsets while performing optimal transport computations by making them much easier and structured.

\section{Implementation Details}
Here, we describe miscellaneous details pertaining to the experiments reported in Section 5 of the main paper. 

\vspace{4pt}
\noindent \textbf{Optimal Transport Computation.}
We use the Python Optimal Transport Library (POT) to conduct our experiments. To keep the computational cost in check, we use a stratified representative set of $5000$ samples from the train sets to calculate the Wasserstein distance (since it involves extracting the source and target latent). This makes our method tractable and practical. We perform stratified sampling to follow a class-balanced approach, i.e. we sample the images inversely proportional to their class frequencies in the train set. Also, we standardize all three terms in OSBORN to avoid the dominance of any term on the others. 

% We employ a Frobenius norm regularized OT solver for a self-supervised setting with the regularization term set to $1 \times e^{-1}$. 

\noindent \textbf{Input Data.} In the case of classification tasks, we resize the input images to $224 \times 224$, and in the case of semantic segmentation, we resize them to $256 \times 256$ (for computational feasibility). %We resize the images to make the experiments computationally less expensive.
Since semantic segmentation is a dense prediction task with a high computational cost, %involving dense predictions. For instance, in a single image of $256 \times 256$ dimension, there are $65,536$ pixel predictions to be made, and hence the overall computational cost to calculate different transferability metrics increases. To tackle this, 
we follow the strategy mentioned in \cite{eleep} and sample $1000$ pixels from an image. Considering class imbalances in semantic segmentation datasets, we sample pixels inversely proportionally to the frequency of their class categories in the target dataset, similar to what MS-LEEP performed in their experiments.

\section{Balancing Three Components of OSBORN}
To study further on importance of each component of OSBORN, we conducted experiments by completely removing one of the terms and reporting the resulting correlations/results in Table \ref{table:comparison_table1}. The analysis demonstrates, interestingly, that the inclusion of the $W_C$ term significantly improves correlation scores. Our metric includes domain difference ($W_D$) and task difference ($W_T$) besides the model cohesiveness term ($W_C$). While selecting models from the source pool, our objective is not just minimizing the model disagreement via ($W_C$) but the entire metric. Through the interplay and equilibrium of these three components, model collapse is prevented.

% \underline{\textit{Min. ($W_C$) leads to model collapse?}} %As shown in our studies (in Fig A1 and here), 

% our objective extends beyond just minimizing the model disagreement term ($W_C$). We also consider the domain difference ($W_D$) and task difference ($W_T$) as essential factors. 

\end{document}